\documentclass[10pt]{article} 
\usepackage[preprint]{tmlr}

\usepackage{amsmath,amsfonts,bm}

\def\eqref#1{equation~\ref{#1}}

\def\1{\bm{1}}

\DeclareMathAlphabet{\mathsfit}{\encodingdefault}{\sfdefault}{m}{sl}
\SetMathAlphabet{\mathsfit}{bold}{\encodingdefault}{\sfdefault}{bx}{n}

\DeclareMathOperator*{\argmin}{arg\,min}

\usepackage{graphicx} 
\usepackage{amsmath}
\usepackage{amssymb}
\usepackage{amsthm}
\usepackage{xcolor}
\usepackage{hyperref}
\usepackage{fullpage}
\usepackage{tikz}
\usepackage{enumitem}
\usetikzlibrary{positioning}
\usepackage{multirow}

\usepackage{booktabs}

\newtheorem{lemma}{Lemma}

\theoremstyle{definition}
\newtheorem{definition}{Definition}

\theoremstyle{remark}

\usepackage{wrapfig}

\usepackage{mathtools}

\usepackage{caption}
\usepackage{subcaption}

\usepackage{algorithm}
\usepackage{algpseudocode}

\usepackage{hyperref}
\usepackage{url}

\usepackage{dsfont}
\usetikzlibrary{positioning}

\title{Score-based Metropolis-Hastings for Fractional Langevin Algorithms}

\author{\name Ahmed Aloui \email ahmed.aloui@duke.edu \\
      \addr Duke University
      \AND
      \name Junyi Liao \email junyi.liao@duke.edu\\
      \addr Duke University
      \AND
      \name Ali Hasan \email ali.hasan@morganstanley.com\\
      \addr Morgan Stanley 
      \AND
      \name Jose Blanchet \email jose.blanchet@stanford.edu \\
      \addr Stanford University 
      \AND
      \name Vahid Tarokh \email vahid.tarokh@duke.edu \\
      \addr Duke University}

\def\month{MM}  
\def\year{YYYY} 
\def\openreview{\url{https://openreview.net/forum?id=XXXX}} 

\begin{document}

\maketitle

\begin{abstract}
Sampling from heavy-tailed and multimodal distributions is challenging when neither the target density nor the proposal density can be evaluated, as in $\alpha$-stable L\'evy-driven fractional Langevin algorithms. While the target distribution can be estimated from data via score-based or energy-based models, the $\alpha$-stable proposal density and its score are generally unavailable, rendering classical density-based Metropolis--Hastings (MH) corrections impractical. Consequently, existing fractional Langevin methods operate in an unadjusted regime and can exhibit substantial finite-time errors and poor empirical control of tail behavior. We introduce the Metropolis-Adjusted Fractional Langevin Algorithm (MAFLA), an MH-inspired, fully score-based correction mechanism. MAFLA employs designed proxies for fractional proposal score gradients under isotropic symmetric $\alpha$-stable noise and learns an acceptance function via Score Balance Matching. We empirically illustrate the strong performance of MAFLA on a series of tasks including combinatorial optimization problems where the method significantly improves finite time sampling accuracy over unadjusted fractional Langevin dynamics.
\end{abstract}

\section{Introduction}

Sampling from heavy-tailed and multimodal distributions remains a fundamental challenge in Bayesian inference and generative modeling, especially when classical MCMC algorithms struggle to explore complex geometries~\citep{roberts1996exponential,robert2004monte,neal2011mcmc}. Fractional Langevin samplers driven by symmetric $\alpha$-stable L\'evy processes have recently emerged as a powerful alternative, offering long jumps and improved exploration in heavy-tailed regimes ~\citep{csimcsekli2017fractional,yoon2023score,shariatianheavy,wang2025fractional}. However, implementing a classical density-based Metropolis--Hastings (MH) ~\citep{metropolis1953equation} correction for these samplers is essentially impossible in practice, since neither the target density nor the $\alpha$-stable proposal density admits a tractable closed form~\citep{nolan2020univariate}. While $p$ can be represented by an energy- or score-model, the conditional $\alpha$-stable proposal density remains intractable. As a consequence, existing fractional Langevin algorithms necessarily operate in an unadjusted regime; following~\citet{csimcsekli2017fractional}, we refer to these methods collectively as the \emph{Fractional Unadjusted Langevin Algorithm} (FULA). 

Density-free approaches cannot provide a density-based MH-inspired correction in the fractional setting. State-of-the-art generative modeling techniques rely on score-based models~\citep{vincent2011connection,song2019generative,song2020score}, which learn only the score $\nabla \log p(x)$ rather than the log-density itself, making it infeasible to directly evaluate the density ratio required by the MH acceptance rule. While it is in principle possible to model the target density via an energy-based model~\citep{{salimans2021should}}, this does not resolve the MH-inspired correction problem, since the proposal density remains unavailable. Independently of how the target distribution is modeled, classical MH requires access to the proposal density $q(x' \mid x)$ in order to compute the acceptance probability \[
a(x',x) = \min\Bigg\{1,\,
\frac{p(x')\,q(x \mid x')}{p(x)\,q(x' \mid x)}
\Bigg\}.
\]
However, except for a few special cases (e.g., Cauchy, Gaussian) when $q$ is an $\alpha$-stable distribution, its density does not admit a closed-form expression and is typically characterized only via its Fourier transform.

The MH correction is particularly crucial in fractional Langevin dynamics, where $\alpha$-stable noise produces occasional long jumps that amplify finite-step discretization bias and strongly affect tail behavior~\citep{roberts1996exponential}. Without this correction, FULA can lead to pronounced distortions in empirical tail statistics and mixture proportions, with these effects becoming more pronounced in higher dimensions. Moreover, as with unadjusted algorithms in general (both classical Langevin and fractional Langevin methods), the sampler becomes overly sensitive to the choice of step size when no MH correction is applied. Even moderate deviations in the step size and discretizations order can lead to degraded finite-time accuracy or numerical instability, further underscoring the motivation of an adjustment mechanism in the fractional setting~\citep{hodgkinson2021implicit}.

Our key idea is that an MH-inspired adjustment for fractional Langevin dynamics can be carried out \emph{without ever evaluating any density}. Instead of relying on intractable target and proposal densities, we enforce a gradient form of detailed balance: we approximate the proposal scores using only the target score and the location-family structure of isotropic symmetric $\alpha$-stable distributions, and we learn an acceptance function that satisfies this gradient constraint via \emph{Score Balance Matching} (SBM)~\citep{aloui2024score}. This yields a fully score-based Metropolis Adjusted Fractional Langevin Algorithm (MAFLA). Our contributions are:

\begin{itemize}
    \item We derive computable, density-free approximations of fractional Langevin proposal scores by exploiting the location-family structure of isotropic symmetric $\alpha$-stable distributions.
    
    \item We adapt Score Balance Matching (SBM) to fractional Langevin dynamics, enabling the learning of an acceptance function that enforces detailed balance at the gradient level without evaluating any density.
    
    \item We show that MAFLA significantly improves finite-time sampling for heavy-tailed targets, addressing the instability of unadjusted fractional Langevin dynamics.
    \item We apply MAFLA to combinatorial optimization via continuous relaxations, showing consistent improvements on challenging problems such as MaxCut and minimum vertex cover, where enhanced exploration from fractional dynamics and stability from Metropolis adjustment jointly lead to higher-quality solutions.
\end{itemize}

\section{Related Work}
\label{sec:related_work}
\textbf{Fractional Langevin Dynamics and L\'evy-driven SDEs.} The development of fractional dynamics for sampling is rooted in statistical physics and probability theory, long before their adoption in machine learning. Early work by \citet{eliazar2003levy} established L\'evy-driven Langevin systems as principled models for anomalous diffusion, while \citet{panloup2008levy} developed the mathematical foundations for approximating invariant measures of L\'evy-driven SDEs. Building on these results, \cite{csimcsekli2017fractional} formally introduced Fractional Langevin Monte Carlo (FLMC) to the machine learning community, and proposed approximating the non-local drift with a scaled local gradient, which leads to the Fractional Unadjusted Langevin algorithm (FULA). This “unadjusted” framework prompted several extensions: \citet{simsekli2020fractional} incorporated momentum to obtain Fractional Underdamped Langevin Dynamics (FULD), and recent generative modeling approaches, including L\'evy-driven score-based models \citep{yoon2023score} and Denoising L\'evy Probabilistic Models (DLPM) \citep{shariatian2024denoising}, leverage L\'evy processes to better capture heavy-tailed structure. Beyond sampling and generative modeling, fractional dynamics have also been adapted to optimization; for instance, \citet{wang2025fractional} demonstrated improved escape behavior in combinatorial optimization tasks.

\textbf{Score-based generative modeling.} Score-based modeling originates from the classical score matching framework of \cite{hyvarinen2005estimation}, later linked to denoising autoencoders by \cite{vincent2011connection}, and further extended to scalable high-dimensional learning through sliced score matching \citep{song2020sliced}. Building on these foundations, score-based generative modeling was developed through noise-conditional score networks \citep{song2019generative} and unified via stochastic differential equations in the reverse-time diffusion framework \citep{song2020score}. Recent work has extended score-based generative modeling to L\'evy-driven SDEs (\citealp{yoon2023score}; \citealp{shariatian2024denoising}).

\textbf{Metropolis–Hastings, Langevin Dynamics, and Score-Based MCMC.} Markov chain Monte Carlo (MCMC) methods originate from the Metropolis–Hastings (MH) algorithm, which enables sampling from complex distributions via an acceptance–rejection correction (\citealp{metropolis1953equation}; \citealp{hastings1970monte}). Incorporating gradient information into proposals led to Langevin-based methods, most notably the Metropolis-adjusted Langevin algorithm (MALA), which improves efficiency while preserving exact invariance of the target distribution (\citealp{roberts1996exponential}; \citealp{robert2004monte}). Building on this idea, score-based MCMC has emphasized the score function as the primary object driving sampling dynamics (\citealp{song2019generative}; \citealp{song2020score}). More recently, \cite{sjoberg2023mcmc} and \cite{aloui2024score} revisited MH-style corrections and introduced MH-like acceptance rules derived directly from score information, enabling MH-style corrections. Moreover, discrete time approximations and weak approximations for jump processes have been studied in the probability literature. For example ~\citet{Mijatovic2014Markov} analyze Markov chain approximations to L\'evy processes and establish sharp convergence rates of transition densities. Related approximations to scale functions for spectrally negative L\'evy processes have also been developed~\citep{Mijatovic2013ScaleFunctions}. In addition, L\'evy processes have been extended to smooth manifolds using Marcus SDEs and generator characterizations~\citep{Mijatovic2021LevyManifolds}, which provides a rigorous foundation for jump driven dynamics on nonlinear state spaces. However, none of these papers propose an MH corrected L\'evy Langevin sampler or work directly in the score-based setting.

\section{Background}
\label{sec:background}


\subsection{Fractional Langevin Algorithm}
\label{sec:fractional_langevin}

\begin{wrapfigure}{r}{0.4\textwidth}
    \vspace{-60pt}
     \centering 
    \includegraphics[width=\linewidth]{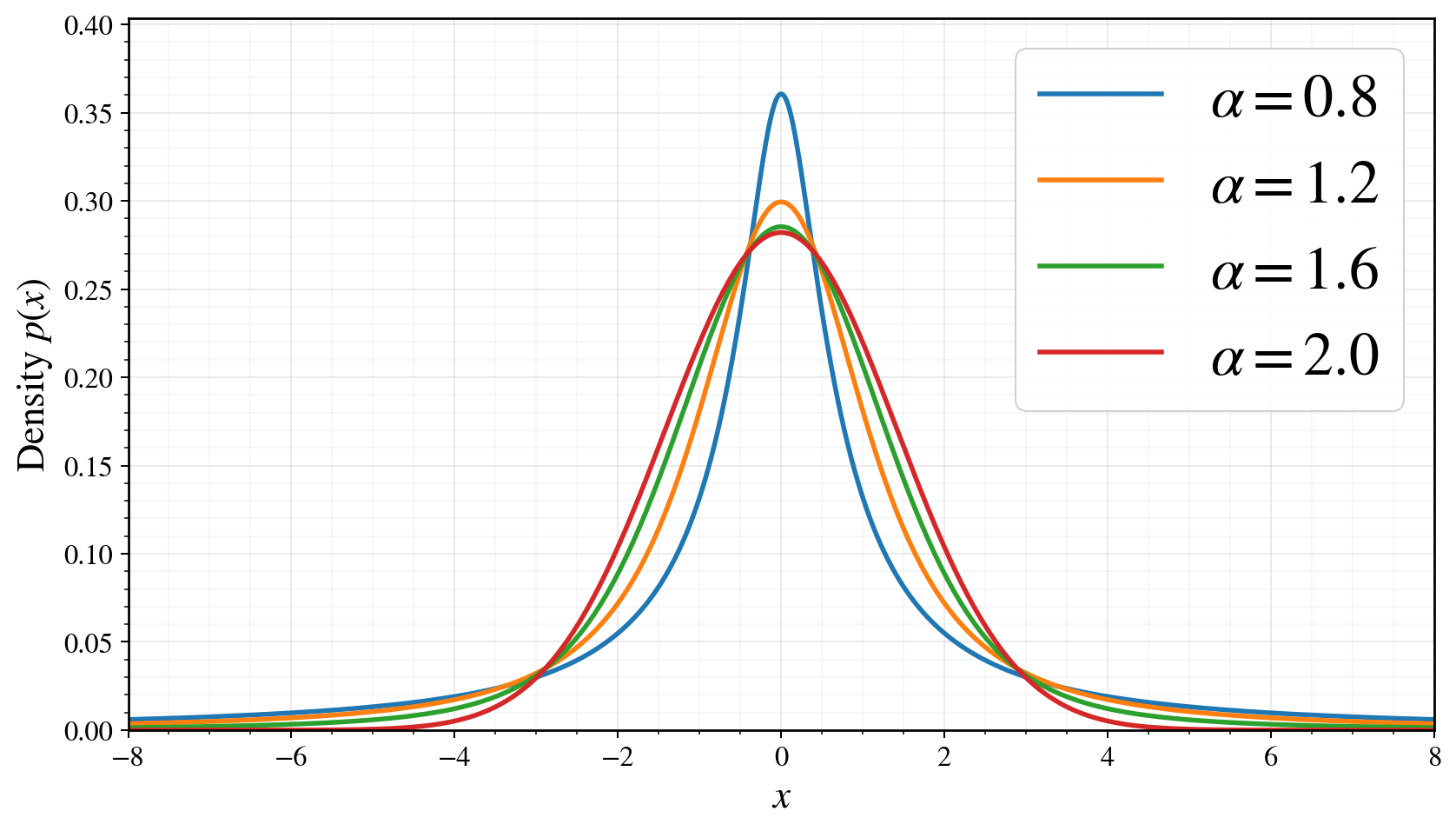}
    \caption{Illustration of symmetric $\alpha$-stable distributions for different tail indices $\alpha$. Smaller $\alpha$ yields heavier tails and higher peak near the origin.}
    \label{fig:alpha-stable}
    \vspace{-10pt}
\end{wrapfigure}

\textbf{Symmetric $\alpha$-stable distributions.} Let $\alpha \in (0,2]$. A real-valued random variable \( X \) is said to follow a \emph{symmetric $\alpha$-stable distribution}, denoted as \( X \sim \mathcal{S}\alpha\mathcal{S}(\sigma) \), if its characteristic function is
\[
\mathbb{E}[e^{i u X}] = \exp\left(-\sigma^\alpha |u|^\alpha\right), \qquad u \in \mathbb{R},
\]

where \( \sigma > 0 \) is a scale parameter. Special cases include Gaussian ($\alpha=2$) and Cauchy ($\alpha=1$); otherwise closed forms are rare. Smaller $\alpha$ yields heavier tails (Fig.~\ref{fig:alpha-stable}).


   

\textbf{Tail behavior and moments.} 

Symmetric $\alpha$-stable laws have finite moments only for orders $<\alpha$; in particular, $\mathrm{Var}(X)=\infty$ for $\alpha<2$.


In higher dimensions, a random vector $L \in \mathbb{R}^d$ is said to be symmetric $\alpha$-stable if every one-dimensional projection $\langle u, L \rangle$ is univariate symmetric $\alpha$-stable for all $u \in \mathbb{R}^d$.
The characteristic function again uniquely specifies the distribution, while the density is typically intractable. In this work, we restrict attention to the isotropic case, where the characteristic function takes the form
$\mathbb{E}[e^{i\langle u,L\rangle}] = \exp(-\sigma^\alpha \|u\|^\alpha)$.

\textbf{L\'evy-driven stochastic differential equations.} Let \( (L_t)_{t \geq 0} \) be a L\'evy process taking values in \( \mathbb{R}^d \), thus $L_t$ has  stationary and independent increments and c\`adl\`ag sample paths. A \emph{L\'evy-driven stochastic differential equation} (or L\'evy-type SDE) is formally written as:
\[
\mathrm{d}X_t = b(X_{t-})\, \mathrm{d}t + \sigma(X_t)\, \mathrm{d}L_t,
\]
where:
\begin{itemize}[itemsep=0pt]
    \item \( b : \mathbb{R}^d \to \mathbb{R}^d \) is the drift vector field,
    \item \( \sigma : \mathbb{R}^d \to \mathbb{R}^{d \times d} \) is the diffusion coefficient,
    \item \( L_t \) is a L\'evy noise, often symmetric and $\alpha$-stable in the context of fractional dynamics, and
    \item $X_{t-}$ denotes the left limit of the process $(X_{t})_{t\geq 0}$.
\end{itemize}

When \( L_t \) is a \emph{symmetric $\alpha$-stable process} with \( \alpha \in (0,2] \), we denote it by $(L_t^\alpha)_{t\geq 0}$~\citep{bertoin1996levy,sato2001basic}. It satisfies:
\begin{itemize}[itemsep=0pt]
    \item[(i)] $L_0^\alpha = 0$ almost surely.
    \item[(ii)] For $0 < t_1 < \dots < t_N$, the increments $(L_{t_n}^\alpha - L_{t_{n-1}}^\alpha)_{n=1}^N$ are independent.
    \item[(iii)] The increment $L_t^\alpha - L_s^\alpha$ is stationary and follows $L_t^\alpha - L_s^\alpha \sim \mathcal{S}\alpha\mathcal{S}(|t - s|^{1/\alpha})$ for $s < t$.
    \item[(iv)] $(L_t^\alpha)$ has stochastically continuous sample paths: for all $\delta > 0$ and $s \geq 0$, 
    \[
    \mathbb{P}(|L_t^\alpha - L_s^\alpha| > \delta) \to 0 \quad \text{as } t \to s.
    \]
\end{itemize}
\textbf{Fractional derivatives.} Let $f: \mathbb{R}^d \rightarrow \mathbb{R}$ belong to the Schwartz class $\mathcal{S}(\mathbb{R}^d)$. The \textit{Riesz fractional derivative} \citep{riesz1949integrale} of order $\gamma>0$ is defined as:
\[
D^{\gamma} f(x) = \mathcal{F}^{-1}\{|\omega|^\gamma \hat{f}(\omega)\},
\]
where $\mathcal{F}$ denotes the Fourier transform and $\hat{f} = \mathcal{F}f$, with
\[
(\mathcal{F}f)(\omega) = \int_{\mathbb{R}^d} f(x) e^{-i2\pi\langle \omega, x\rangle} dx.
\]
\textbf{Fractional Langevin Algorithm.} Introduced in \citep{csimcsekli2017fractional}, the fractional Langevin algorithm is a generalization of the Langevin algorithm to L\'evy-driven dynamics. For a target density $p$ on $\mathbb{R}^d$, the drift term is defined as follows:
\[
b(x) = \frac{D^{\alpha - 2}\left(p(x)\nabla_x \log p(x)\right)}{p(x)}.
\]
In practice, this nonlocal drift is computationally intractable. Following \citet{csimcsekli2017fractional}, we approximate the fractional drift by a scaled standard score. 

\subsection{Score Matching}\label{sec:sm}
Score-based models are grounded in \emph{score matching}~\citep{hyvarinen2005estimation}, which estimates the gradient of the log data density by minimizing the Fisher divergence between the estimated and true gradients of the log data distribution, defined as follows.

\begin{definition}[Fisher Divergence]
    The Fisher divergence \(D_{\nabla}(p_1 \| p_2)\) between two probability distributions \( p_1 \) and \( p_2 \) is defined as:
    \[
     D_{\nabla}(p_1 \| p_2) = \mathbb{E}_{x \sim p_1} \left[ \|\nabla_x \log p_1(x) - \nabla_x \log p_2(x) \|^2 \right].
    \]
\end{definition}
Therefore, given a hypothesis class $\mathcal{S} \subset \{s: \mathcal{X} \rightarrow \mathbb{R}^d\}$, score matching seeks $s^*\in \mathcal{S}$ that minimizes the Fisher divergence to the data distribution $p$: 
\begin{equation}
\label{eq:fisher_loss}
s^* \in \argmin_{s\in \mathcal{S}} \mathbb{E}_{x \sim p} \left[ \|\nabla_x \log p(x) - s(x) \|^2 \right]
\end{equation}
Since the score function is not directly observed, \citet{hyvarinen2005estimation} proved that minimizing in the loss function in Equation \eqref{eq:fisher_loss} is equivalent to minimizing the following loss function:
\begin{equation}
\label{eq:hyv_loss}
    \mathbb{E}_{x \sim p} \left[ \frac{1}{2} \| s(x)\|^2 + \operatorname{tr}\left(\nabla_x s(x)\right) \right].
\end{equation}
This formulation eliminates explicit dependence on the unknown data density, but requires computing the trace of the Jacobian of $s$, which can be expensive in high dimensions. To alleviate this issue, several practical variants have been proposed, including sliced score matching \citep{song2020sliced} and denoising score matching \citep{vincent2011connection,meng2021estimating}.

\begin{definition}[Sliced Score Matching]
    Sliced score matching (SSM) replaces the Jacobian trace in \eqref{eq:hyv_loss} with a stochastic estimator by projecting onto random directions \( v \sim \mathcal{N}(0, I_d) \):
    \begin{equation}
        \mathbb{E}_{x \sim p, v \sim \mathcal{N}(0, I_d)} \left[ \frac{1}{2} \| s(x) \|^2 + v^\top \nabla_x s(x) v \right].
    \end{equation}
\end{definition}

\begin{definition}[Denoising Score Matching]
     Given a noise distribution \( p_\sigma(\tilde{x} | x) \), denoising score matching minimizes the objective
    \begin{equation}
        \mathbb{E}_{x \sim p, \tilde{x} \sim p_\sigma(\tilde{x} | x)} \left[ \|\nabla_{\tilde{x}} \log p_\sigma(\tilde{x} | x) - s(\tilde{x})\|^2 \right].
    \end{equation}
\end{definition}

In this work, we adopt sliced score matching to learn an approximation to the score function of the target distribution from observed heavy-tailed samples, as it provides an efficient and scalable alternative to exact score matching in high dimensions.

\subsection{Metropolis-Hastings}\label{sec:mh}

The Metropolis--Hastings (MH) algorithm \citep{metropolis1953equation, hastings1970monte} constructs a Markov chain using an \emph{acceptance function} \( a: \mathcal{X} \times \mathcal{X} \rightarrow [0,1] \), which defines the probability \( a(x', x) \) of transitioning from the current state \( x \in \mathcal{X} \) to a proposed state \( x' \in \mathcal{X} \). A common choice of the acceptance function is given by:
\begin{equation}
\label{eqn:ratio_acceptance}
 a(x', x) = \min \left\{ 1, \frac{p(x') q(x | x')}{p(x) q(x' | x)} \right\}  
\end{equation}
where \( p(x) \) is the target distribution and \( q(x | x') \) is the proposal distribution. The MH algorithm proceeds as follows:
\begin{itemize}[itemsep=-2pt]
    \item \textbf{Initialize:} Choose an initial state \( x_1 \in \mathcal{X} \), a proposal distribution \( q(x'|x) \), and the number of iterations \( T \).
    \item \textbf{For} \( t = 1, \dots, T \):
    \begin{itemize}[itemsep=0pt]
        \item Propose \( x' \sim q(x'|x_t) \).
        
        \item Compute the acceptance probability \(a(x', x_t)\)~\eqref{eqn:ratio_acceptance}.
        
        \item Draw \( u \sim U(0, 1) \). If \( u \leq a(x', x_t) \), accept $x^\prime$ and set \( x_{t+1} = x' \); otherwise, set \( x_{t+1} = x_t \).
    \end{itemize}
\end{itemize}
Note that evaluating the acceptance probability in \eqref{eqn:ratio_acceptance} requires access to the target density $p(x)$ up to a normalizing constant. 

\textbf{Detailed Balance.} Although the acceptance function is often defined by Equation~\eqref{eqn:ratio_acceptance}, a sufficient condition for ensuring convergence to the target distribution is that the acceptance ratio satisfies the \emph{detailed balance} condition:
\begin{equation}
\label{eqn:detailed_balance}
\forall x,x' \in \mathcal{X}, \quad \frac{a(x', x)}{a(x, x')} = \frac{p(x') \,q(x \mid x')}{p(x)\, q(x' \mid x)}.
\end{equation}
Here, $p(x)$ denotes the target probability density, and $q(x^\prime|x)$ denotes the proposal transition density.
In particular, any acceptance function satisfying the detailed balance condition above ensures that the Markov chain admits $p(x)$ as its unique stationary distribution.

\section{Approximating Proposal Scores }
\label{sec:frac-proposal}

We now describe the fractional Langevin proposals used in MAFLA. We assume access to the target score
$
s_{\theta}(x) \approx \nabla_x \log p(x),
$
e.g., given by a trained deep neural network modeling the score function. Following~\cite{csimcsekli2017fractional}, we approximate the nonlocal fractional drift:
\begin{equation}
    \tilde b(x) = c_\alpha \nabla_x \log p(x), 
    \qquad 
    c_\alpha = \frac{\Gamma(\alpha - 1)}{\Gamma(\alpha/2)^2},
    \label{eq:frac-drift-approx}
\end{equation}
for $\alpha\in(1,2]$, where $\Gamma(\cdot)$ is the Gamma function. 

\textbf{L\'evy-driven proposal and location-family structure.} Let $\mathcal{S}\alpha\mathcal{S}(1)$ denote the standard symmetric $\alpha$-stable distribution on $\mathbb{R}^d$ with unit scale. Given a step size $\tau>0$ and an $\alpha$-stable random vector $\xi\sim\mathcal{S}\alpha\mathcal{S}(1)$, we define the L\'evy-driven proposal:
\begin{equation}
    x' = x + \tau\, \tilde b(x) + \tau^{1/\alpha} \xi, 
    \qquad \xi \sim \mathcal{S}\alpha\mathcal{S}(1).
    \label{eq:frac-proposal}
\end{equation}
Conditioned on $x$, the proposal $x'$ follows a symmetric $\alpha$-stable \emph{location family} with density
\begin{equation}
    q(x'|x)
    =\tau^{-d/\alpha}\,
    f_\alpha\!\left(
        \frac{x'-x-\tau\tilde b(x)}{\tau^{1/\alpha}}
    \right),
    \label{eq:q-density-formal}
\end{equation}
where $f_\alpha$ denote the (intractable) density of the standard symmetric $\alpha$-stable distribution. For most values of $\alpha$, the functional form of $f_\alpha$ has no closed form expression \footnote{ While $f_\alpha$ admits an integral representation via the inverse Fourier transform (e.g., a Hankel–Bessel integral for radial characteristic functions), no closed-form is available for most $\alpha$.}. Therefore, neither $q(x' \mid x)$ nor its logarithm can be evaluated explicitly, and hence we cannot compute the classical MH acceptance ratio.

\textbf{Fractional proposal scores.} Intuitively, while the proposal density itself is unavailable, its score can be approximated through the structure of the $\alpha$-stable location family and fractional calculus, which enables gradient-based corrections without explicit density evaluation. In our setting, even if though the density $f_\alpha$ is intractable, its associated \emph{fractional score} admits a closed form expression. Define 
\[
r = x'-x-\tau\tilde b(x), \qquad 
r_{\mathrm{rev}} = x-x'-\tau\tilde b(x').
\]
From~\cite{yoon2023score}, for $\alpha \in (1,2]$, the fractional score of a symmetric isotropic $\alpha$-stable location family with scale $\tau^{1/\alpha}$ is given by
\[
\begin{aligned}
    S^{(\alpha)}_q(x'|x)
& = \frac{D^{\alpha - 2}\left(q(x'|x)\nabla_{x'} \log q(x'|x)\right)}{q(x'|x)} \\
& = -\frac{1}{\gamma^{\alpha - 1}} \cdot \frac{1}{\alpha} \cdot \left( \frac{x' - x - \tau \tilde b(x)}{\gamma} \right),
\end{aligned}
\]
where $\gamma=\tau^{1/\alpha}$. We can simplify this to
\[
S^{(\alpha)}_q(x'|x)
= -\frac{1}{\alpha}\,\frac{x'-x-\tau\tilde b(x)}{\tau}
= -\frac{1}{\alpha\tau}\,r.
\]
According to ~\citep{csimcsekli2017fractional}, we approximate the log-gradient by a scaled fractional score, \(\nabla_{x'} \log q(x'|x) \approx c_\alpha^{-1} S^{(\alpha)}_q(x'|x)\),
to obtain
\begin{equation}
    \nabla_{x'} \log q(x'|x) \approx -\kappa\,r,
    \qquad
    \kappa = \frac{1}{\alpha\,c_\alpha\,\tau}.
    \label{eq:grad-xprime-q}
\end{equation}
For $\alpha\in(1,2]$ and small step size $\tau$, the scaled fractional score provides a first-order approximation to the ordinary score in the local regime where proposals concentrate, with higher-order errors vanishing as $\tau\to 0$. This approximation is sufficient for enforcing detailed balance at the gradient level in the SBM objective.

To find the gradient with respect to \(x\), define
\(z = (x'-m(x))/\tau^{1/\alpha}\) with the location map \(m(x) = x + \tau \tilde b(x)\). Then
\[
\nabla_x z = -\tau^{-1/\alpha}(I+\tau J_{\tilde b}(x)),
\]
where \(J_{\tilde b}(x)\) is the Jacobian of \(\tilde b\). By the chain rule,
\[
\begin{aligned}
\nabla_x \log q(x'|x)
& = (\nabla_x z)^\top \nabla_z \log f_\alpha(z) \\
& = -(I+\tau J_{\tilde b}(x))^\top \nabla_{x'} \log q(x'|x).
\end{aligned}
\]
Substituting \eqref{eq:grad-xprime-q} yields
\[
\nabla_x \log q(x'|x)
\approx \kappa\,(r+\tau J_{\tilde b}(x)^\top r).
\]
Applying the same argument to the reverse proposal \(q(x|x')\)
and exchanging \(x\leftrightarrow x'\) gives the complete set of approximate proposal gradients:
\begin{equation}
\begin{aligned}
\nabla_{x'} \log q(x'|x) &\approx -\kappa\, r, \\
\nabla_{x} \log q(x'|x) &\approx \kappa\,(r+\tau J_{\tilde b}(x)^\top r), \\
\nabla_{x} \log q(x|x') &\approx -\kappa\, r_{\mathrm{rev}}, \\
\nabla_{x'} \log q(x|x') &\approx \kappa\,(r_{\mathrm{rev}}+\tau J_{\tilde b}(x')^\top r_{\mathrm{rev}}).
\end{aligned}
\label{eq:frac-proposal-grads}
\end{equation}

which we refer to as the \emph{proposal scores}. These expressions are computable from samples and the drift approximation~\eqref{eq:frac-drift-approx}, without ever evaluating the density $f_\alpha$.

\begin{figure*}[t]
    \centering
    \begin{subfigure}[t]{0.33\textwidth}
        \centering
        \includegraphics[width=0.8\linewidth]{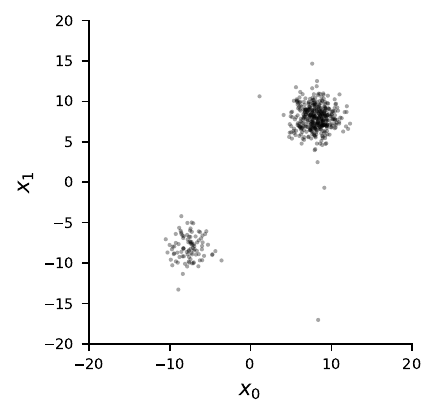}
        \subcaption{Target data. Mixture weights $(0.2, 0.8)$.}
        \label{fig:data-mixture}
    \end{subfigure}\hfill
    \begin{subfigure}[t]{0.33\textwidth}
        \centering
        \includegraphics[width=0.8\linewidth]{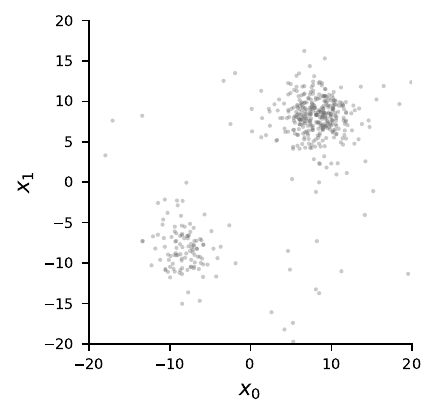}
        \subcaption{FULA. Empirical weights $(0.26, 0.74)$.}
        \label{fig:fula-mixture}
    \end{subfigure}\hfill
    \begin{subfigure}[t]{0.33\textwidth}
        \centering
        \includegraphics[width=0.8\linewidth]{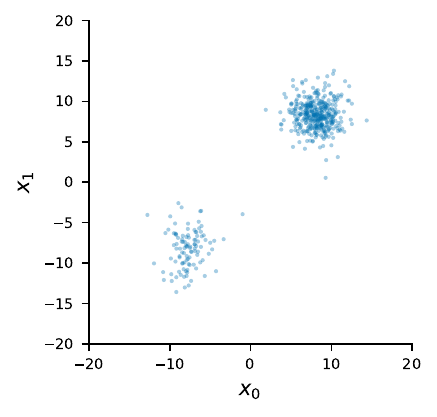}
        \subcaption{MAFLA. Empirical weights $(0.21, 0.79)$.}
        \label{fig:mafla-mixture}
    \end{subfigure}
    \caption{
    Qualitative comparison on a 2D $\alpha$-stable mixture with weights $(0.2, 0.8)$ for $\alpha = 1.95$.
    MAFLA produces samples that visually match both modes and mixture weights more closely than FULA, which oversamples the left mode and exhibits noisier samples.
    }
    \label{fig:mixture-scatter}
\end{figure*}
\begin{figure*}[t]
    \centering
    \begin{minipage}{0.34\linewidth}
        \centering
        \includegraphics[width=\linewidth]{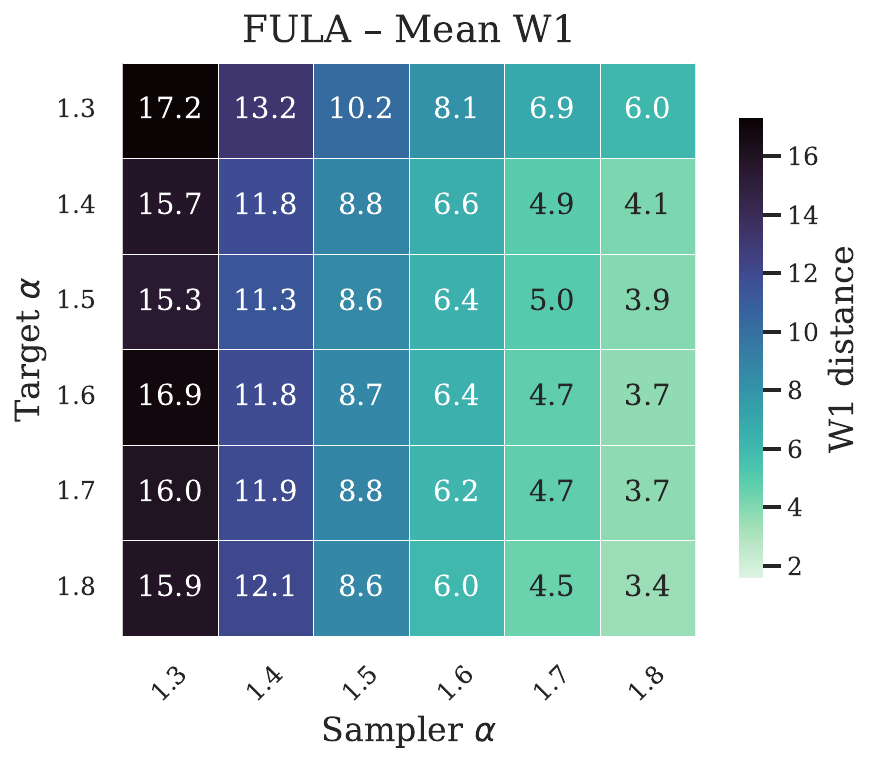}
    \end{minipage}
    \begin{minipage}{0.32\linewidth}
        \centering
        \includegraphics[width=\linewidth]{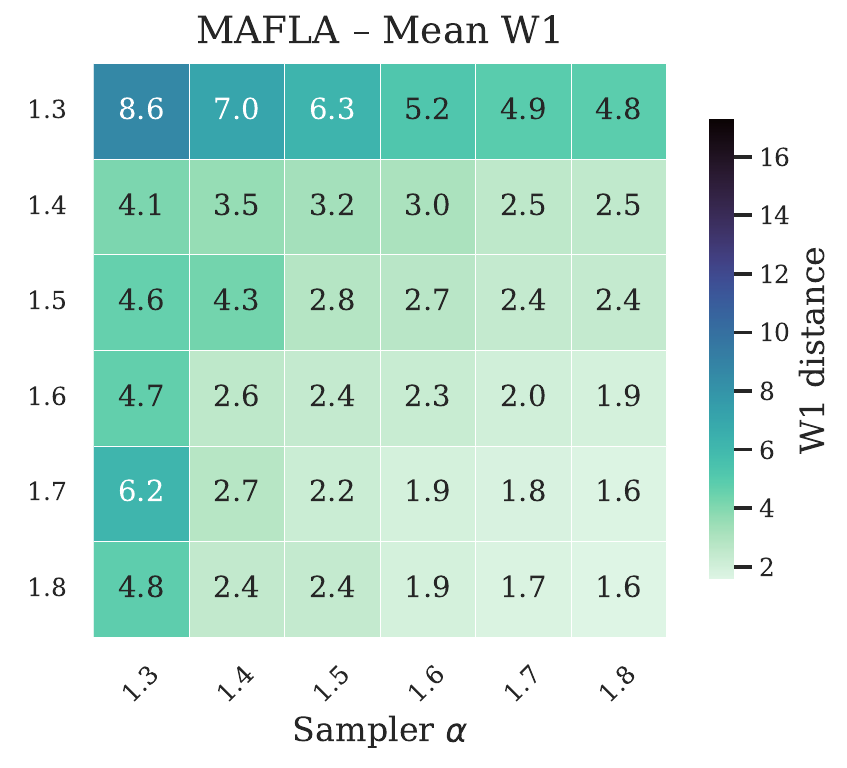}
    \end{minipage}\hfill
    \begin{minipage}{0.32\linewidth}
        \centering
        \includegraphics[width=\linewidth]{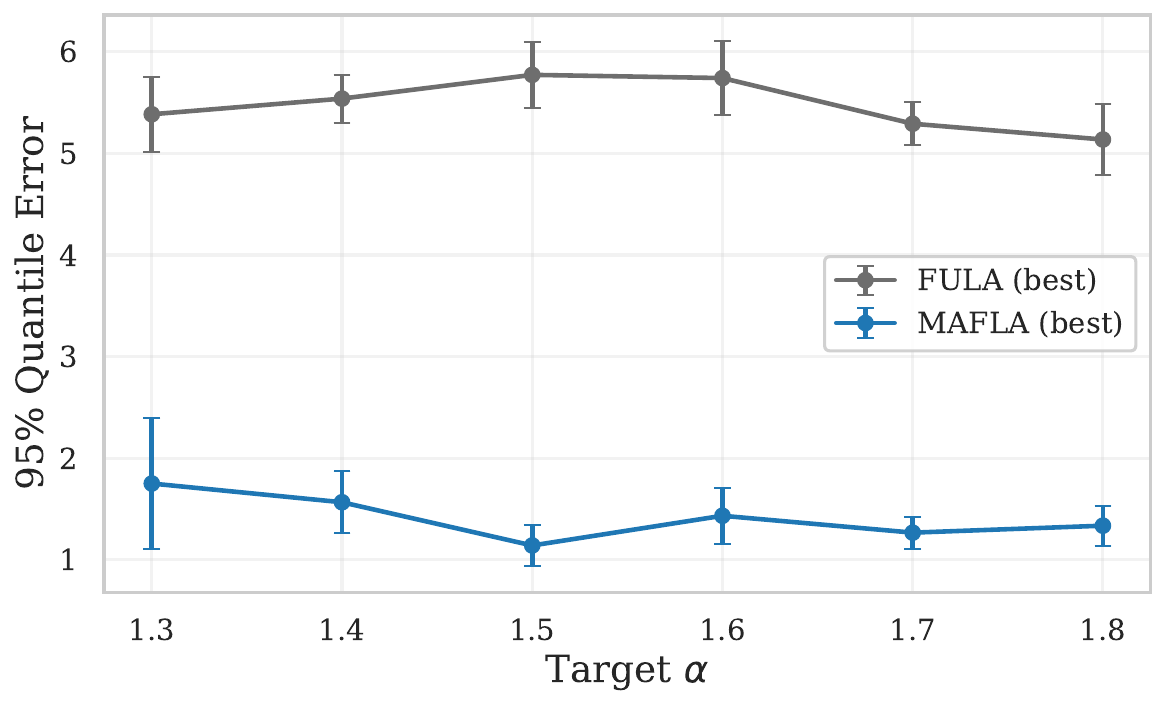}
    \end{minipage}
    \caption{
        Effect of the proposal stability index.
        \textbf{Left:} Mean $W_1$ distance for FULA and MAFLA over a grid of
        target $\alpha_{\text{tgt}}$ and proposal $\alpha_{\text{prop}}$.
        \textbf{Right:} Best $95\%$ quantile error (mean $\pm$ std over runs)
        as a function of $\alpha_{\text{tgt}}$, comparing FULA (gray) and
        MAFLA (blue).
    }
    \label{fig:alpha-grid}
\end{figure*}

\begin{algorithm}[t]
\caption{Parallel Score-based Metropolis-Adjusted Fractional Langevin Algorithm (MAFLA)}
\label{alg:parallel_smafla}
\begin{algorithmic}[1]
\Require Number of particles $N$, initial states $(x_0^{(i)})_{i=1}^N \subset \mathbb{R}^d$, step size $\tau>0$, stability index $\alpha\in(1,2]$, number of iterations $T$, score network $s_\theta(x)$, acceptance network $a_\phi(x',x)$.
\State Compute the scaling constant $c_\alpha = \Gamma(\alpha - 1)/\Gamma(\alpha/2)^2$.
\For{$t=0$ to $T-1$}
    \State \textbf{Drift computation (batched):} $\tilde b^{(i)} \gets c_\alpha\,s_\theta\big(x_t^{(i)}\big)$ for all $i = 1,\dots,N$.
    \State \textbf{Sample L\'evy noise (batched):} draw $\xi_t^{(i)} \sim \mathcal{S}\alpha\mathcal{S}(1)$ independently for all $i = 1,\dots,N$.
    \State \textbf{Propose (batched):} $x'^{(i)} \gets x_t^{(i)} + \tau\,\tilde b^{(i)} + \tau^{1/\alpha}\xi_t^{(i)}$ for all $i$.
    \State \textbf{Acceptance probabilities (batched):} $\alpha_t^{(i)} \gets a_\phi\big(x'^{(i)},x_t^{(i)}\big)$ for all $i$.
    \State Draw $u^{(i)} \sim \mathrm{Unif}(0,1)$ independently for all $i$.
    \For{$i = 1$ to $N$}
        \If{$u^{(i)} < \alpha_t^{(i)}$}
            \State $x_{t+1}^{(i)} \gets x'^{(i)}$\Comment{accept}
        \Else
            \State $x_{t+1}^{(i)} \gets x_t^{(i)}$\Comment{reject}
        \EndIf
    \EndFor
\EndFor
\State \Return $\big(x_t^{(i)}\big)_{t=0}^T{}_{,\,i=1}^N$.
\end{algorithmic}
\end{algorithm}


\section{Score Balance Matching}

Following \citep{aloui2024score}, we now explain how to learn an acceptance function $a(x',x)$ using only target and proposal scores.

\textbf{Gradient form of detailed balance. }
\label{sec:grad-detailed-balance}
Let $p$ denote the target distribution with score $\nabla \log p(x)$ and $q$ a proposal kernel. A Metropolis--Hastings transition with proposal $q$ and acceptance function $a:\mathcal{X}^2\to[0,1]$ satisfies the detailed balance condition: i.e., for every $ x, x' \in \mathbb{R}^d$
\begin{equation}
p(x)\, q(x' \mid x)\, a(x, x') = p(x')\, q(x \mid x')\, a(x', x)
\label{eq:dbalance}
\end{equation}
Assuming $a(x',x)>0$ on the support of $p(x)\,q(x'|x)$, taking the logarithm and gradient with respect to $(x,x')$ yields
\begin{equation}
\begin{aligned}
& \nabla\log p(x) + \nabla\log q(x'|x) + \nabla\log a(x,x') = \nabla\log p(x') + \nabla\log q(x|x') + \nabla\log a(x',x),
\end{aligned}
\label{eq:grad-detailed-balance}
\end{equation}
where $\nabla = (\nabla_x, \nabla_{x'})$. This condition depends only on target scores, proposal score gradients, and acceptance score gradients, without requiring evaluation of any density. The equivalence between \eqref{eq:dbalance} and \eqref{eq:grad-detailed-balance} is shown in Appendix~\ref{sec:appendix-grad-detailed-balance}.

\textbf{Score Balance Matching (SBM).}
Using the gradient detailed balance condition \eqref{eq:grad-detailed-balance}, we learn an acceptance function $a(x',x)\in(0,1]$ by minimizing the \emph{Score Balance Matching} objective
\begin{equation}
\mathcal{L}_{2}(a)\!=\!\mathbb{E}
\left[\big\|
\nabla\log a(x',x)\!-\!\nabla\log a(x,x')\!-\!\Delta_p\!-\!\Delta_q
\big\|_2^2\right],
\label{eq:sbm_l2}
\end{equation}
where
\[
\begin{aligned}
\Delta_p &:= \nabla\log p(x')-\nabla\log p(x),\\
\Delta_q &:= \nabla\log q(x|x')-\nabla\log q(x'|x).
\end{aligned}
\]

Minimizing \eqref{eq:sbm_l2} encourages $a$ to satisfy detailed balance while avoiding evaluation of $p$ and $q$ (up to normalizing constants).

\textbf{Fractional variant of the SBM objective.} Motivated by the $\alpha$-stable dynamics (Section~\ref{sec:frac-proposal}), we introduce a fractional variant of the SBM objective, which is an $\alpha$-norm version of the loss:
\begin{equation}
\mathcal{L}_{\alpha}(a)\!
=\!\mathbb{E}\left[\!
\left\|
\nabla\log a(x',x)\!-\!\nabla\log a(x,x')\!-\!\Delta_p\!-\!\Delta_q
\right\|_\alpha^\alpha\right].
\label{eq:sbm_alpha}
\end{equation}
where $\|\cdot\|_\alpha$ is the $\ell_\alpha$ norm.
We then combine the two objectives into a single training criterion:
\begin{equation}
\label{eq:sbm_objective}
\mathcal{L}(a)
= \mathcal{L}_{2}(a) + \lambda_\alpha\, \mathcal{L}_{\alpha}(a),
\end{equation}
with $\lambda_\alpha \ge 0$ controlling the contribution of the fractional loss component.

\paragraph{Entropy regularization.}
To avoid degenerate near-zero acceptance, we add an entropy penalty as in~\citet{aloui2024score}, we augment the loss with an entropy term:
\begin{align}
\mathcal{L}_{r}(a)
&= \mathcal{L}(a) + \lambda\, \mathbb{E}\!\left[ H(a(x',x)) \right],
\label{eq:sbm_ent}
\end{align}
where
\begin{equation}
\begin{aligned}
H(a(x',x))
& =
a(x',x)\,\log a(x',x) +
\big(1-a(x',x)\big)
\log\!\big(1-a(x',x)\big),
\end{aligned}
\end{equation}
and $\lambda \ge 0$ controls the strength of entropy regularization. Minimizing this term encourages acceptance probabilities with higher entropy, preventing collapse to trivial or uninformative solutions. In practice, SBM enforces gradient balance only approximately, since both the target score and the proposal scores are themselves approximations. We do not provide theoretical bounds linking these approximations to stationarity error, and therefore rely on empirical validation. Establishing such guarantees is an important direction for future work.

\section{Metropolis-Adjusted Fractional Langevin}

We now present the practical algorithm combining fractional Langevin proposals and an SBM-trained acceptance function.

\textbf{Training the acceptance network.} We parameterize the acceptance function by a neural network $a_\phi:\mathbb{R}^d\times\mathbb{R}^d\to(0,1)$, typically of the form $a_\phi(x',x) = \sigma(g_\phi(x',x)),$
where $g_\phi$ is an unconstrained neural network, i.e, $g_\phi: \mathbb{R}^d\times\mathbb{R}^d\to\mathbb{R}$, learning the acceptance logits, and $\sigma$ denotes the logistic sigmoid.

\begin{wrapfigure}{r}{0.4\textwidth}
    \vspace{-12pt}
    \centering
    \includegraphics[width=\linewidth]{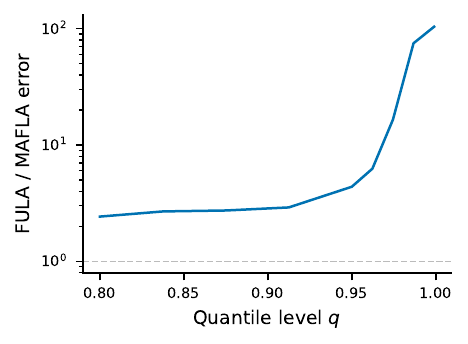}
    \caption{
    Quantile error ratio. MAFLA uniformly dominates as the ratio is greater than $1$.
    }
    \label{fig:mixture-quantile-improvement}
    \vspace{-20pt}
\end{wrapfigure}

\textbf{Curriculum interpolation for SBM training.} To stabilize training, we generate training pairs using a convex interpolation between data samples and proposal samples: \(x'^{(i)} = \eta v' + (1 - \eta) \tilde{x}\), with \(\tilde{x} \sim p\) and \(v' \sim q(\cdot\mid \tilde{x})\), where \(\eta \in [0,1]\) controls how much \(x'\) deviates from the data distribution towards the proposal. In the first epochs of training, we begin with smaller values of \(\eta\) to favor data from the true distribution, and then gradually increase \(\eta\) to estimate the acceptance function around the proposal region. This approach is motivated by the fact that the score function may not be well-estimated outside of the data region, and starting with smaller \(\eta\) helps prevent biasing the training.

In practice, we replace the target score $\nabla\log p$ with the learned score $s_\theta(x)$, and the proposal score gradients $\nabla\log q$ with the approximate proposal scores~\eqref{eq:frac-proposal-grads}.

Algorithm~\ref{alg:parallel_smafla} presents MAFLA (parallel multistart). We use the parallel form in all experiments for scalability and fair comparison with FULA. For completeness, the sequential version of MAFLA (Algorithm~\ref{alg:smafla}) is presented in the appendix.

\section{Experiments}

We evaluate MAFLA against FULA on heavy-tailed sampling and combinatorial optimization.

\textbf{Heavy-tailed mixtures.} 
We first sample from a 2-component $\alpha$-stable mixture with $\alpha=1.95$ (Fig.~\ref{fig:mixture-scatter}), where MH correction remains crucial even near the Gaussian regime. MAFLA improves mixture weight recovery and reduces error from $W_1=4.87$ (FULA) to $0.52$, and the $0.99$-quantile error from $80.33$ to $0.80$ (Fig.~\ref{fig:mixture-quantile-improvement}).

\textbf{Proposal index misspecification.} We vary the proposal stability $\alpha_{\text{prop}}$ across targets $\alpha_{\text{tgt}}$ using a 4-D symmetric $\alpha$-stable location family with mean $(1,2,3,4)$. Over the $(\alpha_{\text{tgt}},\alpha_{\text{prop}})$ grid, MAFLA consistently yields smaller mean $W_1$ and improves the best $95\%$ quantile error per $\alpha_{\text{tgt}}$ by $\sim 3$--$5\times$, while being less sensitive to $\alpha_{\text{prop}}$ misspecification (Fig.~\ref{fig:alpha-grid}).

\textbf{Step-size sensitivity.} On a 4-D $\alpha$-stable mixture with $\alpha=1.5$, we vary $\tau$ over three orders of magnitude and report $W_1$, $q_{0.95}$, and $q_{0.99}$ errors. Errors increase with $\tau$ for both methods, but MAFLA consistently attains lower $W_1$ and tail errors, indicating improved robustness to step-size misspecification (Fig.~\ref{fig:tau-ablation}).

\textbf{Scaling with dimension. }
For a symmetric bimodal $\alpha$-stable mixture with modes $\pm \mathbf{1}_d$ and weights $0.6/0.4$ at $\alpha=1.9$, we vary $d\in\{8,12,16,20,24,28,32\}$. Both samplers degrade with $d$, but MAFLA remains better across all dimensions, reducing $W_1$ by $\sim 20$--$35\%$ and improving $95\%/99\%$ tail errors (Fig.~\ref{fig:dim-ablation}).

\textbf{Combinatorial optimization.} Using continuous relaxations with explicit score functions (Appendix~\ref{sec:codetail}), we apply ULA/FULA/MALA/MAFLA to MaxCut and minimum vertex cover on ER and BA graphs with $N\in\{64,256,512,1024\}$. For MaxCut, MAFLA achieves the best mean and best-found cuts across all settings (Table~\ref{tab:maxcut_ba_er_compact}). For vertex cover, MAFLA yields the smallest mean and best cover sizes and the lowest uncovered-edge ratio before greedy decoding (Table~\ref{tab:vertex_cover} in the Appendix), indicating improved solution quality and feasibility. These gains can be attributed to the complementary roles of fractional dynamics and Metropolis adjustment: the heavy-tailed fractional noise encourages broader exploration and facilitates escape from poor local minima, while the learned acceptance function selectively filters unlikely proposals, stabilizing the sampling dynamics and enforcing detailed balance.

\section{Conclusion}
We propose MAFLA, an MH-inspired $\alpha$-stable L\'evy sampler for settings where both target and proposal densities are intractable. MAFLA replaces density ratios with a gradient-form detailed balance objective, combining score-based fractional drift approximations, closed-form proposal-score proxies for symmetric isotropic $\alpha$-stable location families, and a learned acceptance function trained via Score Balance Matching. Across heavy-tailed targets, MAFLA improves over FULA in $W_1$ and tail-quantile error, and is more robust to $\alpha$, step size, and dimension; it also performs competitively on MaxCut and vertex cover, outperforming ULA, FULA, and MALA.



\begin{table*}[h]
\centering
\caption{MaxCut results on Barabási--Albert (BA) and Erdős--Rényi (ER) graphs. Reported as mean\,$\pm$\,std over runs, plus best cut found. The superscript $^\uparrow$ means higher values are preferred.}
\label{tab:maxcut_ba_er_compact}
\small
\setlength{\tabcolsep}{6pt}

\begin{tabular}{c l c c c c c c}
\toprule
\multirow{2}{*}{$N$} & \multirow{2}{*}{Sampler}
& \multicolumn{3}{c}{Barabási–Albert ($m=2$)}  & \multicolumn{3}{c}{Erdős-Rényi ($p=0.1$)}  \\
\cmidrule(lr){3-5}\cmidrule(lr){6-8}
&& Energy ($\mu\pm\sigma$) & Cut$^{\uparrow}$ ($\mu\pm\sigma$) & Best$^{\uparrow}$
 & Energy ($\mu\pm\sigma$) & Cut$^{\uparrow}$ ($\mu\pm\sigma$) & Best$^{\uparrow}$ \\
\midrule

\multirow{4}{*}{64}
 & ULA   & $-1.20\pm1.16$ & $67.97\pm5.58$ & 90
 & $-3.89\pm2.02$ & $214.50\pm9.68$ & 253 \\
 & FULA  & $-2.93\pm1.35$ & $74.46\pm5.46$ & 92
 & $-6.22\pm2.49$ & $222.21\pm10.04$ & 256 \\
 & MALA  & $-0.99\pm1.05$ & $67.23\pm5.48$ & 87
 & $-3.23\pm1.89$ & $212.32\pm9.70$ & 249 \\
 & MAFLA & $-3.66\pm1.38$ & $\mathbf{77.72\pm5.70}$ & \bf 100
 & $-6.85\pm2.55$ & $\mathbf{225.03\pm10.41}$ & $\mathbf{264}$ \\
\midrule

\multirow{4}{*}{256}
 & ULA   & $-1.22\pm1.16$ & $265.48\pm11.22$ & 311
 & $-15.24\pm4.02$ & $3391.87\pm38.36$ & 3530 \\
 & FULA  & $-4.88\pm1.40$ & $295.13\pm11.43$ & 335
 & $-24.83\pm4.94$ & $3455.98\pm39.79$ & 3621 \\
 & MALA  & $-1.02\pm1.09$ & $264.25\pm11.16$ & 305
 & $-12.46\pm3.75$ & $3373.21\pm38.58$ & 3519 \\
 & MAFLA & $-7.31\pm1.36$ & $\bf 316.59\pm11.33$ & \bf 361
 & $-34.35\pm7.05$ & $\mathbf{3536.37\pm60.00}$ & \bf 3728 \\
\midrule

\multirow{4}{*}{512}
 & ULA   & $-1.25\pm1.16$ & $526.54\pm15.83$ & 579
 & $-30.08\pm5.70$ & $13432.87\pm76.49$ & 13749 \\
 & FULA  & $-6.11\pm1.41$ & $583.23\pm16.33$ & 633
 & $-49.45\pm7.01$ & $13615.99\pm80.01$ & 13881 \\
 & MALA  & $-1.01\pm1.08$ & $524.27\pm15.84$ & 592
 & $-24.46\pm5.21$ & $13379.40\pm75.47$ & 13697 \\
 & MAFLA & $-9.04\pm1.36$ & $\bf 620.95\pm16.20$ & \bf 685
 & $-63.60\pm7.96$ & $\bf 13780.58\pm91.54$ & \bf 14093 \\
\midrule

\multirow{4}{*}{1024}
 & ULA   & $-1.27\pm1.18$ & $1045.70\pm22.65$ & 1129
 & $-59.38\pm8.02$ & $53142.28\pm152.80$ & 53661 \\
 & FULA  & $-7.38\pm1.39$ & $1148.52\pm23.01$ & 1237
 & $-97.93\pm9.66$ & $53659.39\pm156.13$ & 54212 \\
 & MALA  & $-1.00\pm1.09$ & $1041.65\pm22.99$ & 1116
 & $-48.22\pm7.60$ & $52989.74\pm154.92$ & 53598 \\
 & MAFLA & $-10.02\pm1.39$ & $\bf 1197.17\pm23.54$ & \bf 1278
 & $-142.51\pm11.13$ & $\bf 54396.39\pm182.21$ & \bf 55244 \\
\bottomrule
\end{tabular}
\end{table*}

\begin{figure*}[t]
    \centering
    \begin{minipage}{0.3\linewidth}
        \centering
        \includegraphics[width=\linewidth]{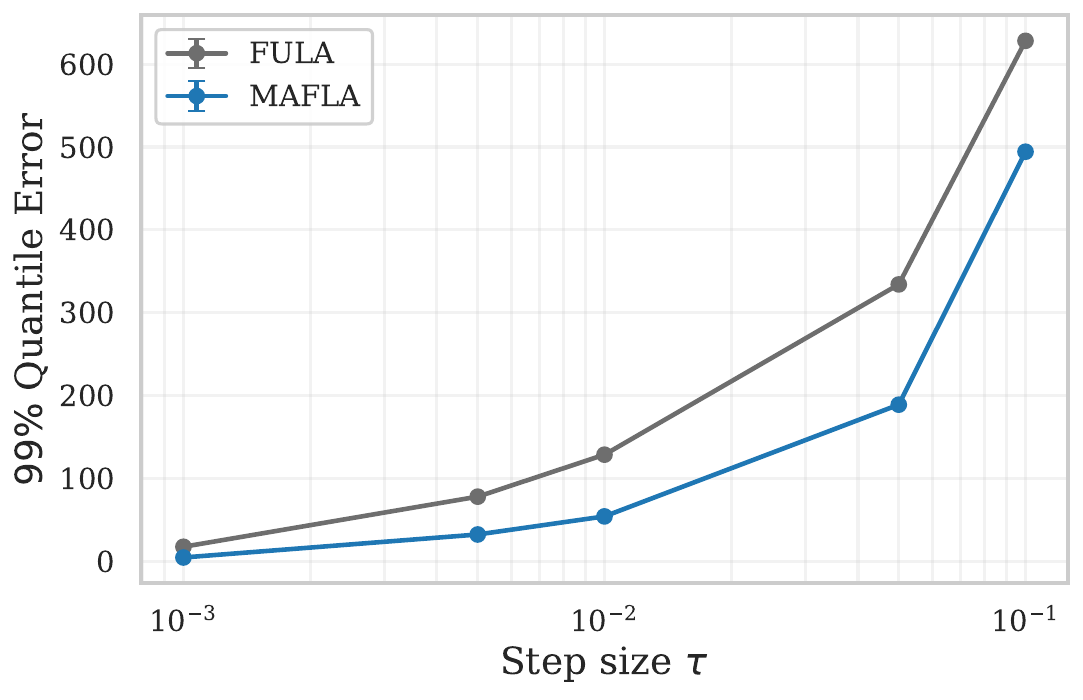}
    \end{minipage}
    \hfill
    \begin{minipage}{0.3\linewidth}
        \centering
        \includegraphics[width=\linewidth]{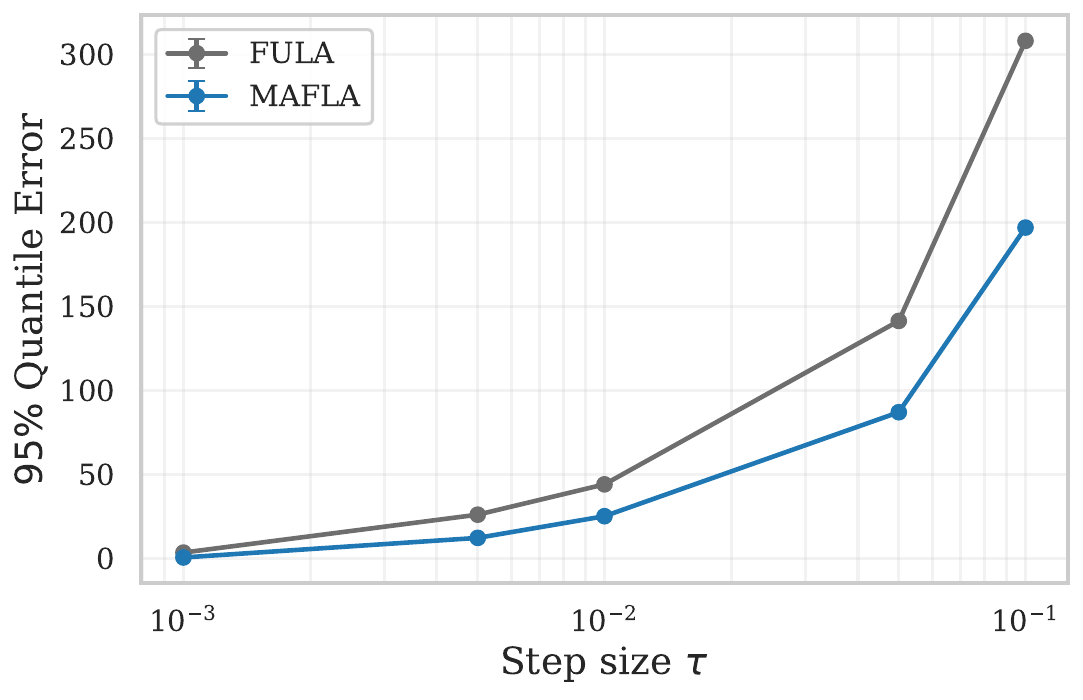}
    \end{minipage}
    \hfill
    \begin{minipage}{0.3\linewidth}
        \centering
        \includegraphics[width=\linewidth]{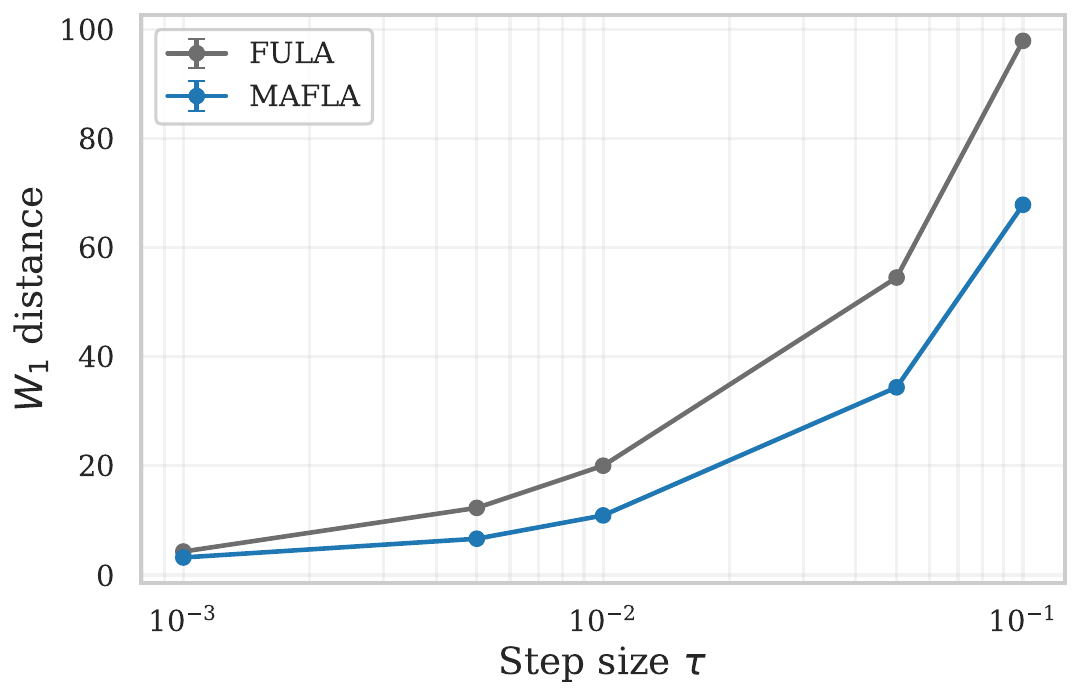}
    \end{minipage}
    \caption{
        Sensitivity to step size~$\tau$ for FULA and MAFLA in a $4$-D $\alpha$-stable mixture, with $\alpha=1.5$.  
        \textbf{Left:} $99\%$ quantile error.  
        \textbf{Center:} $95\%$ quantile error.  
        \textbf{Right:} Wasserstein distance $W_1$.  
        MAFLA exhibits lower errors across all $\tau$ and improves tail behavior.
    }
    \label{fig:tau-ablation}
\end{figure*}

\begin{figure*}[t]
    \centering
    \begin{minipage}{0.3\linewidth}
        \centering
        \includegraphics[width=\linewidth]{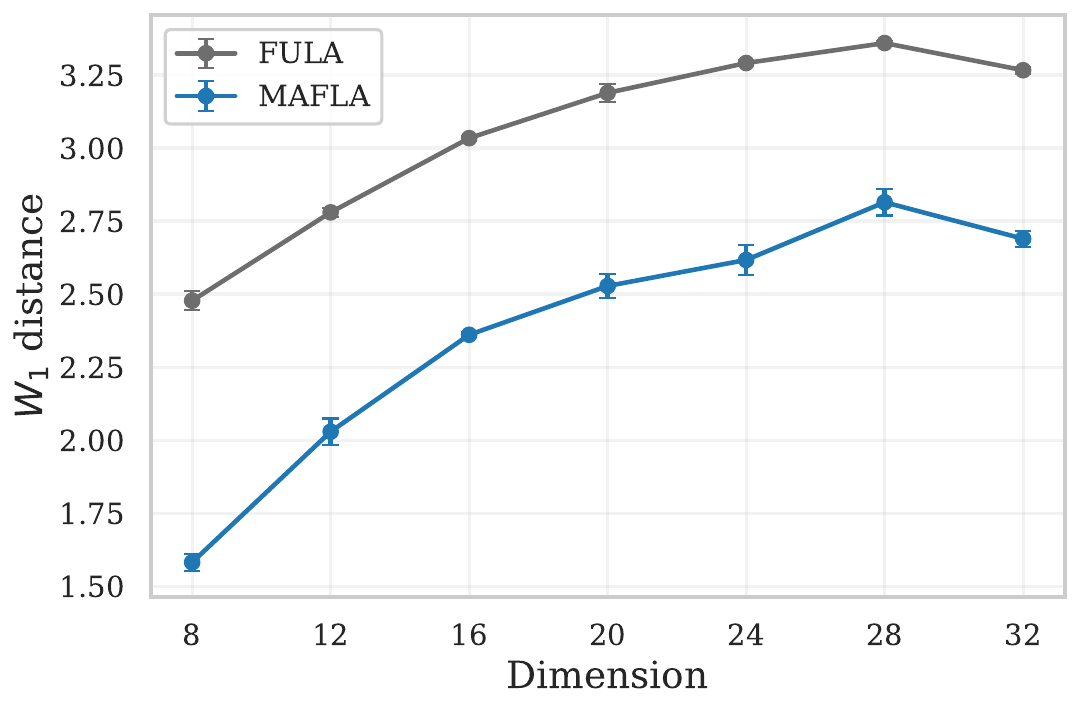}
    \end{minipage}
    \hfill
    \begin{minipage}{0.30\linewidth}
        \centering
        \includegraphics[width=\linewidth]{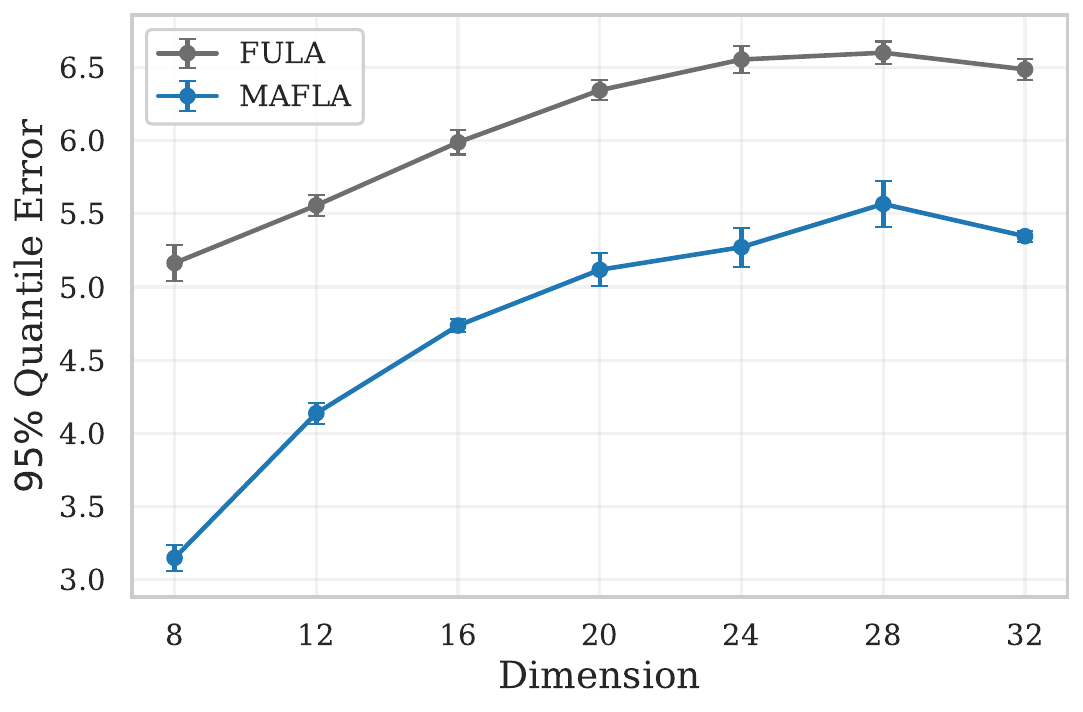}
    \end{minipage}
    \hfill
    \begin{minipage}{0.30\linewidth}
        \centering
        \includegraphics[width=\linewidth]{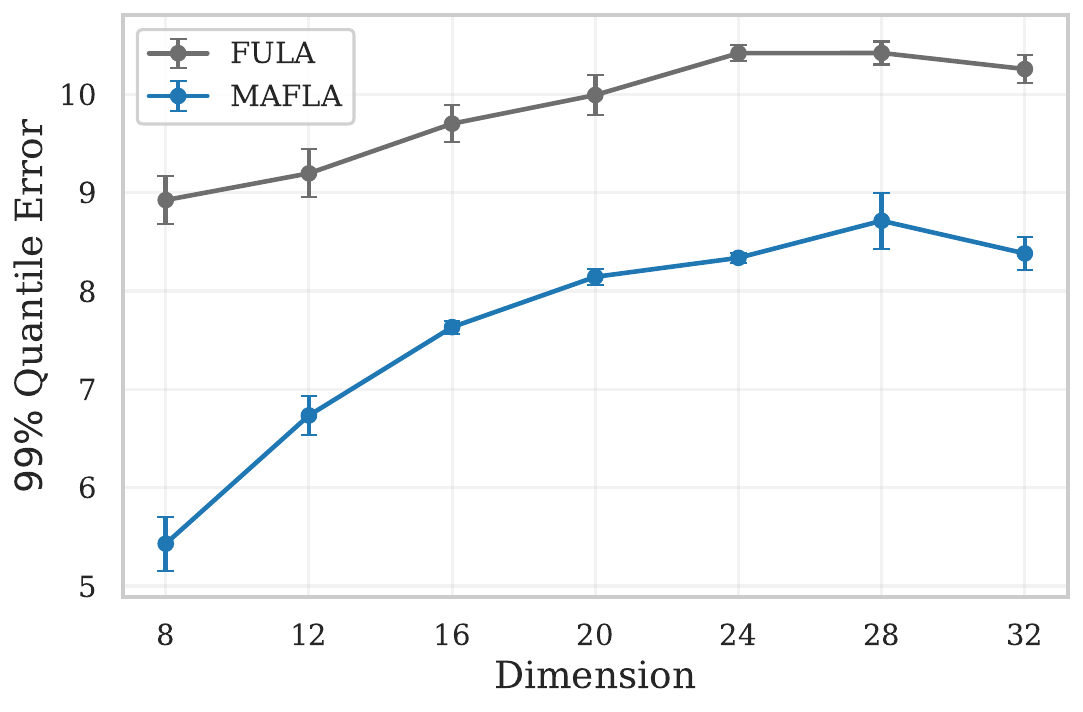}
    \end{minipage}
    \caption{
        Scaling with dimension for FULA and MAFLA on a $d$-dimensional symmetric $\alpha$-stable bimodal mixture with modes at $\pm \mathbf{1}_d$ and mixture weights $0.6/0.4$.
        \textbf{Left:} Wasserstein distance $W_1$.
        \textbf{Center:} $95\%$ quantile error.
        \textbf{Right:} $99\%$ quantile error.
        Curves show mean $\pm$ standard error over three runs.
    }
    \label{fig:dim-ablation}
\end{figure*}

\subsubsection*{Acknowledgments} This work was supported in part by the Air Force Office of Scientific Research under award number FA9550-20-1-0397.

\bibliography{tmlr}

@inproceedings{csimcsekli2017fractional,
  title={Fractional {L}angevin {M}onte {C}arlo: Exploring L{\'e}vy driven stochastic differential equations for markov chain monte carlo},
  author={{\c{S}}im{\c{s}}ekli, Umut},
  booktitle={International Conference on Machine Learning},
  pages={3200--3209},
  year={2017},
  organization={PMLR}
}

@article{yoon2023score,
  title={Score-based generative models with L{\'e}vy processes},
  author={Yoon, Eun Bi and Park, Keehun and Kim, Sungwoong and Lim, Sungbin},
  journal={Advances in Neural Information Processing Systems},
  volume={36},
  pages={40694--40707},
  year={2023}
}

@article{hyvarinen2005estimation,
  title={Estimation of non-normalized statistical models by score matching.},
  author={Hyv{\"a}rinen, Aapo},
  journal={Journal of Machine Learning Research},
  volume={6},
  number={4},
  year={2005}
}

@article{aloui2024score,
  title={Score-based {M}etropolis-{H}astings Algorithms},
  author={Aloui, Ahmed and Hasan, Ali and Dong, Juncheng and Wu, Zihao and Tarokh, Vahid},
  journal={arXiv preprint arXiv:2501.00467},
  year={2024}
}

@article{metropolis1953equation,
  title={Equation of state calculations by fast computing machines},
  author={Metropolis, Nicholas and Rosenbluth, Arianna W and Rosenbluth, Marshall N and Teller, Augusta H and Teller, Edward},
  journal={The journal of chemical physics},
  volume={21},
  number={6},
  pages={1087--1092},
  year={1953},
  publisher={American Institute of Physics}
}

@incollection{robert2004monte,
  title={Monte {C}arlo optimization},
  author={Robert, Christian P and Casella, George},
  booktitle={Monte Carlo Statistical Methods},
  pages={157--204},
  year={2004},
  publisher={Springer}
}

@inproceedings{shariatianheavy,
  title={Heavy-Tailed Diffusion with Denoising {L}{\'e}vy Probabilistic Models},
  author={Shariatian, Dario and Simsekli, Umut and Durmus, Alain Oliviero},
  booktitle={The Thirteenth International Conference on Learning Representations}
}

@inproceedings{salimans2021should,
  title={Should {EBM}s model the energy or the score?},
  author={Salimans, Tim and Ho, Jonathan},
  booktitle={Energy based models workshop-ICLR 2021},
  year={2021}
}

@article{song2019generative,
  title={Generative modeling by estimating gradients of the data distribution},
  author={Song, Yang and Ermon, Stefano},
  journal={Advances in neural information processing systems},
  volume={32},
  year={2019}
}

@article{vincent2011connection,
  title={A connection between score matching and denoising autoencoders},
  author={Vincent, Pascal},
  journal={Neural computation},
  volume={23},
  number={7},
  pages={1661--1674},
  year={2011},
  publisher={MIT Press}
}

@book{nolan2020univariate,
  title={Univariate stable distributions},
  author={Nolan, John P},
  year={2020},
  publisher={Springer}
}

@inproceedings{wang2025fractional,
  title={Fractional {L}angevin Dynamics for Combinatorial Optimization via Polynomial-Time Escape},
  author={Wang, Shiyue and Guo, Ziao and Lu, Changhong and Yan, Junchi},
  booktitle={The Thirty-ninth Annual Conference on Neural Information Processing Systems},
  year={2025}
}

@article{neal2011mcmc,
  title={{MCMC} using Hamiltonian dynamics},
  author={Neal, Radford M and others},
  journal={Handbook of markov chain monte carlo},
  volume={2},
  number={11},
  pages={2},
  year={2011},
  publisher={Chapman and Hall/CRC}
}

@article{meng2021estimating,
  title={Estimating high order gradients of the data distribution by denoising},
  author={Meng, Chenlin and Song, Yang and Li, Wenzhe and Ermon, Stefano},
  journal={Advances in Neural Information Processing Systems},
  volume={34},
  pages={25359--25369},
  year={2021}
}

@inproceedings{song2020sliced,
  title={Sliced score matching: A scalable approach to density and score estimation},
  author={Song, Yang and Garg, Sahaj and Shi, Jiaxin and Ermon, Stefano},
  booktitle={Uncertainty in artificial intelligence},
  pages={574--584},
  year={2020},
  organization={PMLR}
}

@article{song2020score,
  title={Score-based generative modeling through stochastic differential equations},
  author={Song, Yang and Sohl-Dickstein, Jascha and Kingma, Diederik P and Kumar, Abhishek and Ermon, Stefano and Poole, Ben},
  journal={arXiv preprint arXiv:2011.13456},
  year={2020}
}

@article{roberts1996exponential,
  title={Exponential convergence of {L}angevin distributions and their discrete approximations},
  author={Roberts, Gareth O and Tweedie, Richard L},
  year={1996}
}

@article{ortigueira2006riesz,
  title={Riesz potential operators and inverses via fractional centred derivatives},
  author={Ortigueira, Manuel Duarte},
  journal={International Journal of Mathematics and Mathematical Sciences},
  volume={2006},
  number={1},
  pages={048391},
  year={2006},
  publisher={Wiley Online Library}
}

@article{eliazar2003levy,
  title={L{\'e}vy-driven {L}angevin systems: Targeted stochasticity},
  author={Eliazar, Iddo and Klafter, Joseph},
  journal={Journal of statistical physics},
  volume={111},
  number={3},
  pages={739--768},
  year={2003},
  publisher={Springer}
}

@article{panloup2008levy,
author = {Fabien Panloup},
title = {{Recursive computation of the invariant measure of a stochastic differential equation driven by a {L}{\'e}vy process}},
volume = {18},
journal = {The Annals of Applied Probability},
number = {2},
publisher = {Institute of Mathematical Statistics},
pages = {379 -- 426},
year = {2008},
}

@inproceedings{simsekli2020fractional,
  title={Fractional underdamped {L}angevin dynamics: Retargeting sgd with momentum under heavy-tailed gradient noise},
  author={Simsekli, Umut and Zhu, Lingjiong and Teh, Yee Whye and Gurbuzbalaban, Mert},
  booktitle={International conference on machine learning},
  pages={8970--8980},
  year={2020},
  organization={PMLR}
}

@article{shariatian2024denoising,
  title={Denoising {L}{\'e}vy Probabilistic Models},
  author={Shariatian, Dario and Simsekli, Umut and Durmus, Alain},
  journal={arXiv preprint arXiv:2407.18609},
  year={2024}
}

@article{hastings1970monte,
  title={Monte {C}arlo sampling methods using {M}arkov chains and their applications},
  author={Hastings, W Keith},
  year={1970},
  publisher={Oxford University Press}
}

@article{sjoberg2023mcmc,
  title={{MCMC}-Correction of Score-Based Diffusion Models for Model Composition},
  author={Sj{\"o}berg, Anders and Lindqvist, Jakob and {\"O}nnheim, Magnus and Jirstrand, Mats and Svensson, Lennart},
  journal={arXiv preprint arXiv:2307.14012},
  year={2023}
}

@article{riesz1949integrale,
  title={L'int{\'e}grale de {R}iemann-{L}iouville et le probl{\`e}me de {C}auchy},
  author={Riesz, Marcel},
  journal={Acta Mathmatica},
  volume={81},
  number={1},
  pages={1--222},
  year={1949}
}

@article{Mijatovic2014Markov,
  author    = {Aleksandar Mijatovi{\'{c}} and Matija Vidmar and Saul Jacka},
  title     = {Markov Chain Approximations for Transition Densities of {L}{\'{e}}vy Processes},
  journal   = {Electronic Journal of Probability},
  year      = {2014},
  volume    = {19},
  pages     = {1--37},
  doi       = {10.1214/EJP.v19-2208},}

@article{Mijatovic2021LevyManifolds,
  author    = {Aleksandar Mijatovi{\'{c}} and Veno Mramor},
  title     = {L{\'{e}}vy Processes on Smooth Manifolds with a Connection},
  journal   = {Electronic Journal of Probability},
  year      = {2021},
  volume    = {26},
  article   = {132},
  doi       = {10.1214/21-EJP702},
}

@unpublished{Mijatovic2013ScaleFunctions,
  author    = {Aleksandar Mijatovi{\'{c}} and Matija Vidmar and Saul Jacka},
  title     = {Markov Chain Approximations to Scale Functions of {L}{\'{e}}vy Processes},
  note      = {arXiv:1310.1737},
  year      = {2013},}

@article{hodgkinson2021implicit,
  title={Implicit {L}angevin algorithms for sampling from log-concave densities},
  author={Hodgkinson, Liam and Salomone, Robert and Roosta, Fred},
  journal={Journal of Machine Learning Research},
  volume={22},
  number={136},
  pages={1--30},
  year={2021}
}

@book{bertoin1996levy,
  title={L{\'e}vy processes},
  author={Bertoin, Jean},
  volume={121},
  year={1996},
  publisher={Cambridge university press Cambridge}
}

@incollection{sato2001basic,
  title={Basic results on {L}{\'e}vy processes},
  author={Sato, Ken-iti},
  booktitle={L{\'e}vy processes: theory and applications},
  pages={3--37},
  year={2001},
  publisher={Springer}
}
\bibliographystyle{tmlr}

\appendix
\section{Appendix}

\subsection{MAFLA}
Here we include the sequential version of the MAFLA algorithm (Alg.~\ref{alg:smafla}). While it corresponds to the standard Metropolis–Hastings algorithm, in higher dimensions the multi-particle version is more stable and more scalable, as the different particles can run in parallel and we avoid the need for large burn-in and lag parameters that would significantly slow down the sampling procedure. Moreover, the parallel version allows for a fair comparison with FULA, since both methods are run for the same number of steps with the same number of particles.

\begin{algorithm}[h]
\caption{Score-based Metropolis-Adjusted Fractional Langevin Algorithm (MAFLA)}
\label{alg:smafla}
\begin{algorithmic}[1]
\Require Initial state $x_0\in\mathbb{R}^d$, step size $\tau>0$, stability index $\alpha\in(1,2]$, number of iterations $T$, score network $s_\theta(x)$, acceptance network $a_\phi(x',x)$.
\State Compute the scaling constant $c_\alpha = \Gamma(\alpha - 1)/\Gamma(\alpha/2)^2$.
\For{$t=0$ to $T-1$}
    \State \textbf{Drift computation:} $\tilde b(x_t) \gets c_\alpha\,s_\theta(x_t)$.
    \State \textbf{Sample L\'evy noise:} draw $\xi_t \sim \mathcal{S}\alpha\mathcal{S}(1)$.
    \State \textbf{Propose:} $x' \gets x_t + \tau\,\tilde b(x_t) + \tau^{1/\alpha}\xi_t$.
    \State \textbf{Acceptance probability:} $\alpha_t \gets a_\phi(x',x_t)$.
    \State Draw $u\sim\mathrm{Unif}(0,1)$.
    \If{$u < \alpha_t$}
        \State $x_{t+1} \gets x'$\Comment{accept}
    \Else
        \State $x_{t+1} \gets x_t$\Comment{reject}
    \EndIf
\EndFor
\State \Return $(x_t)_{t=0}^T$.
\end{algorithmic}
\end{algorithm}


\subsection{Score Balance Matching}
\label{sec:appendix-grad-detailed-balance}
In this section, we recall a lemma from \citet{aloui2024score} showing that the gradient form of the detailed balance condition introduced in
Section~\ref{sec:grad-detailed-balance} is equivalent to the standard
detailed balance equation.

\begin{lemma}[Gradient Detailed Balance Equivalence]
\label{lem:grad-detailed-balance}
Let $\mathcal{X}\subseteq\mathbb{R}^d$ and assume that the target density
$p:\mathcal{X}\to\mathbb{R}_{+}$, the proposal density
$q(\cdot\mid x):\mathcal{X}\to\mathbb{R}_{+}$, and the acceptance function
$a:\mathcal{X}\times\mathcal{X}\to(0,1]$ are differentiable.
Then the following statements are equivalent for all $x,x'\in\mathcal{X}$:
\begin{itemize}
    \item[(i)] \emph{(Detailed balance)}
    \[
    p(x)\, q(x' \mid x)\, a(x, x')
    =
    p(x')\, q(x \mid x')\, a(x', x).
    \]

    \item[(ii)] \emph{(Gradient detailed balance)}
    \[
    \nabla \log p(x)
    + \nabla \log q(x' \mid x)
    + \nabla \log a(x,x')
    =
    \nabla \log p(x')
    + \nabla \log q(x \mid x')
    + \nabla \log a(x',x),
    \]
\end{itemize}
where \(\nabla=(\nabla_x,\nabla_{x'})\).
\end{lemma}


\subsection{Symmetric Proposal: Score-based Fractional Random Walk}
For completeness we include a discussion about symmetric proposals, e.g., fractional random walk, we prove an equivalence result showing that training any acceptance function (any function that verifies the detailed balance condition) is equivalent to training an energy function. We then confirm empirically that training an energy model directly (with score matching) and performing the MH correction with the learned energy (as the proposal is not needed) has equivalent performance to training a score model and then an acceptance function.

\subsubsection{Equivalence Result}
We begin from the log detailed balance condition, for $a >0 $,
\[
\log \frac{a(x',x)}{a(x,x')} \;=\; \log p(x') - \log p(x) + \log q(x\mid x') - \log q(x'\mid x).
\]
When the proposal is symmetric (e.g.\ Random Walk Metropolis--Hastings), the ratio simplifies to
\[
\log \frac{a(x',x)}{a(x,x')} \;=\; \log p(x') - \log p(x).
\]
We show that any acceptance function satisfying this identity admits a decomposition of the form
\[
a(x',x) \;=\; f(x)\, g(x')\, r(x',x),
\]
where \(f,g:\mathcal{X}\to\mathbb{R}_+\) and \(r:\mathcal{X}\times\mathcal{X}\to\mathbb{R}_+\) is symmetric, i.e.\ \(r(x',x)=r(x,x')\).

Let
\[
\psi(x,x')=\log \frac{a(x',x)}{a(x,x')},\qquad 
h_1(x)=-\log p(x),\quad h_2(x')=\log p(x').
\]
Then the gradient form of detailed balance is
\[
\nabla \psi(x,x')=\nabla h_1(x)+\nabla h_2(x'),
\]
which integrates to
\[
\psi(x,x') = h_1(x)+h_2(x')+C
\]
for some constant \(C\). Writing \(\alpha=e^C\),
\[
\frac{a(x',x)}{a(x,x')} = \alpha\, e^{h_1(x)} e^{h_2(x')}.
\]

We now establish the decomposition formally.

\begin{lemma}\label{lemma:decomposition}
For \(a:\mathcal{X}\times\mathcal{X}\to(0,1]\), the following are equivalent:
\begin{enumerate}
    \item[(i)] There exist functions \(\tilde h_1,\tilde h_2:\mathcal{X}\to\mathbb{R}\) such that
    \[
    \frac{a(x',x)}{a(x,x')} = \exp(\tilde h_1(x))\,\exp(\tilde h_2(x')).
    \]
    \item[(ii)] There exist \(f,g:\mathcal{X}\to\mathbb{R}_+\) and a symmetric \(r:\mathcal{X}\times\mathcal{X}\to\mathbb{R}_+\) such that
    \[
    a(x',x)=f(x)\, g(x')\, r(x',x).
    \]
\end{enumerate}
\end{lemma}

\begin{proof}
\textbf{(i) $\Rightarrow$ (ii).}
Define
\[
r(x',x)=\sqrt{a(x',x)\,a(x,x')},\qquad 
f(x)=e^{\tilde h_1(x)/2},\quad g(x')=e^{\tilde h_2(x')/2}.
\]
A direct calculation shows
\[
a(x',x)=f(x)\,g(x')\,r(x',x).
\]

\textbf{(ii) $\Rightarrow$ (i).}
Define \(\tilde h_1(x)=\log f(x)\) and \(\tilde h_2(x')=\log g(x')\).  
Symmetry of \(r\) yields
\[
\frac{a(x',x)}{a(x,x')}
    =\frac{f(x)g(x')}{f(x')g(x)}
    =\exp(\tilde h_1(x))\,\exp(\tilde h_2(x')),
\]
which gives (i).
\end{proof}

Applying Lemma~\ref{lemma:decomposition} to the detailed balance identity yields
\[
\log\frac{f(x)g(x')}{f(x')g(x)}=\log p(x')-\log p(x).
\]
Taking gradients with respect to \(x\) and \(x'\),
\[
\nabla_x\bigl(\log g(x)-\log f(x)\bigr)=\nabla_x \log p(x),\qquad
\nabla_{x'}\bigl(\log g(x')-\log f(x')\bigr)=\nabla_{x'} \log p(x').
\]
Thus, \(\log g - \log f\) recovers the target energy up to constants, while the symmetric factor \(r\) controls the residual flexibility of the acceptance rule.

This decomposition motivates parameterizing the acceptance function with three neural networks \(f\), \(g\), and a symmetric \(r\), allowing the model to learn the target energy through \(\log f - \log g\) while ensuring valid acceptance probabilities through the symmetric component \(r\).

We therefore summarize a class of fractional-proposal Metropolis–Hastings (MH) algorithms (Fig.~\ref{fig:fractional-mh-scheme}). If the proposal is symmetric, training an energy model is mathematically equivalent to training a score model together with an acceptance function, enabling MH correction. However, for non-symmetric fractional proposals (e.g., Fractional Langevin, Fractional Hamiltonian Monte Carlo), the proposal density is not tractable, and training an energy model alone is not sufficient to perform MH correction.

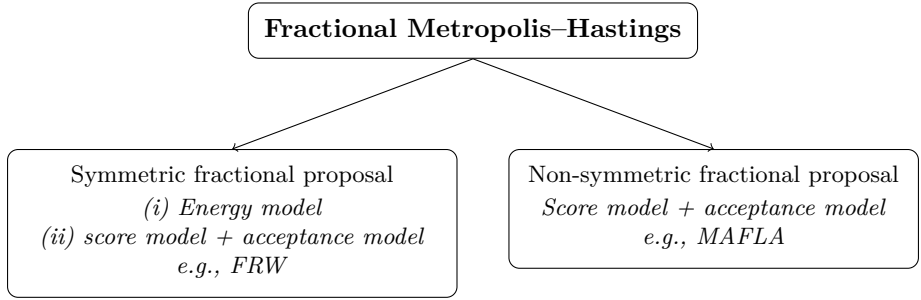
\begin{figure}[h]
    \centering
    \begin{tikzpicture}[
        node distance=1.0cm and 0.2cm,
        every node/.style={font=\small},
        box/.style={
            draw,
            rounded corners,
            align=center,
            minimum width=4.2cm,
            inner sep=6pt
        },
        level/.style={sibling distance=6cm}
    ]

    \node[box, font=\bfseries] (fmh) {Fractional Metropolis--Hastings};

\node[box, below=1.2cm of fmh, xshift=-3.2cm] (sym) {Symmetric fractional proposal\\[2pt]
    \begin{tabular}{c}
    \emph{(i) Energy model}\\
    \emph{(ii) score model + acceptance model}\\
    \emph{e.g., FRW}
    \end{tabular}
};

\node[box, below=1.2cm of fmh, xshift=3.2cm] (nonsym) {Non-symmetric fractional proposal\\[2pt]
    \begin{tabular}{c}
    \emph{Score model + acceptance model}\\
    \emph{e.g., MAFLA}
    \end{tabular}
};

    \draw[->] (fmh.south) -- (sym.north);
    \draw[->] (fmh.south) -- (nonsym.north);

    \end{tikzpicture}
    \caption{Schematic overview of Fractional Metropolis--Hastings Algorithms.}
    \label{fig:fractional-mh-scheme}
\end{figure}

\subsection{Numerical Approximation of the Fractional Drift}
\label{app:drift_approx}

\subsubsection{Riesz Fractional Derivative Approximation}

The fractional drift in our Langevin dynamics is governed by the Riesz fractional derivative of order $\gamma = \alpha - 2$, defined as $\mathcal{D}^{\gamma}f(x) = \mathcal{F}^{-1}\{|\omega|^{\gamma}\hat{f}(\omega)\}$. Due to the non-local nature of this operator, exact evaluation is intractable. To implement this in the Metropolis-Adjusted Fractional Langevin Algorithm (MAFLA), we employ a discretized approximation based on fractional centered differences.

Following \cite{ortigueira2006riesz} and \cite{csimcsekli2017fractional}, the Riesz derivative can be approximated by the limit of the centered difference operator $\Delta_h^\gamma$:
\begin{equation}
    \mathcal{D}^{\gamma}f(x) = \lim_{h \to 0} \frac{1}{h^\gamma} \sum_{k=-\infty}^{\infty} g_{\gamma,k} f(x - kh),
\end{equation}
where $h > 0$ is the discretization step size. The coefficients $g_{\gamma,k}$ are given by the ratio of Gamma functions:
\begin{equation}
    g_{\gamma,k} = (-1)^k \frac{\Gamma(\gamma + 1)}{\Gamma(\frac{\gamma}{2} - k + 1)\Gamma(\frac{\gamma}{2} + k + 1)}.
\end{equation}
For computational tractability, we truncate the infinite summation to a finite window size $K \in \mathbb{N}$, resulting in the \textit{truncated fractional centered difference operator} $\Delta_{h,K}^\gamma$:
\begin{equation}
    \Delta_{h,K}^{\gamma}f(x) := \frac{1}{h^\gamma} \sum_{k=-K}^{K} g_{\gamma,k} f(x - kh).
\end{equation}

\subsubsection{Score-Based Implementation}

In the context of score-based modeling, we do not have access to the unnormalized density $f_\pi(x) = p(x)$. The fractional drift requires the term $\mathcal{D}^{\alpha-2}(p(x) \nabla \log p(x)) / p(x)$. Applying the discrete operator to the density-weighted score $p(x)s(x)$, we obtain:
\begin{equation}
    \tilde{b}_{h,K}(x) \approx \frac{1}{h^{\alpha-2}} \sum_{k=-K}^{K} g_{\alpha-2, k} \frac{p(x-kh)}{p(x)} s_\theta(x-kh),
    \label{eq:score_approx}
\end{equation}
where $s_\theta(x) \approx \nabla \log p(x)$ is the learned score network. The density ratio $\frac{p(x-kh)}{p(x)}$ is estimated by integrating the score function along the path from $x$ to $x-kh$, approximated via the trapezoidal rule.

\subsubsection{Error Analysis and Hyperparameters}

The accuracy of this approximation depends on the interplay between the step size $h$ and the truncation order $K$. Assuming sufficient regularity of the target density, the approximation error is bounded by:
\begin{equation}
    \left| \mathcal{D}^{\gamma}f(x) - \Delta_{h,K}^{\gamma}f(x) \right| = \mathcal{O}\left(h^2 + \frac{1}{hK}\right).
\end{equation}
This bound highlights a critical trade-off:
\begin{itemize}
    \item \textbf{Discretization Error ($h^2$):} Reduces as $h \to 0$, favoring smaller step sizes for local accuracy.
    \item \textbf{Truncation Error ($(hK)^{-1}$):} Increases as $h \to 0$ if $K$ is fixed. This term reflects the heavy-tailed nature of the operator; as the grid becomes finer (smaller $h$), the window $K$ must effectively grow to capture the non-local interactions of the tails.
\end{itemize}
\subsubsection{Ablation Study: Approximation Order \texorpdfstring{$K$}{K} and Drift Step \texorpdfstring{$h$}{h}}
\label{sec:ablation_kh}

While we utilized the robust zeroth-order drift approximation ($K=0$) for the main results, the proposed framework allows for higher-order Riesz derivative approximations ($K \ge 1$) as defined in equation~\eqref{eq:score_approx}. In this section, we investigate the sensitivity of S-MAFLA to the truncation order $K$ and the discretization step size $h$, seeking to characterize the trade-off between approximation fidelity and stability.

We evaluated the sampler on $\alpha$-stable targets with $\alpha \in \{1.2, 1.5, 1.8\}$ and mean $\mu=5.0$, varying $K \in \{0, 1, 3, 5\}$ and $h \in [10^{-3}, 10^{-1}]$. The results, summarized in Figure~\ref{fig:heatmaps}, reveal two key insights:

\begin{itemize}
    \item \textbf{Performance Ceiling of Higher-Order Approximations:} When the discretization step $h$ is finely tuned, higher-order approximations can significantly outperform the robust $K=0$ baseline. For example, in the heavy-tailed regime ($\alpha=1.2$), increasing the order to $K=1$ with $h=10^{-2}$ reduces the Wasserstein-1 error to $3.17$, an improvement over the $4.16$ baseline achieved by $K=0$. Similarly, for $\alpha=1.8$, a fine-tuned model ($K=5, h=10^{-1}$) achieves a Wasserstein-1 distance of $0.81$, almost halving the error of the zeroth-order model ($1.45$).
    
    \item \textbf{Robustness-Accuracy Trade-off:} While higher $K$ values offer a higher performance ceiling, they exhibit greater sensitivity to the choice of $h$. As illustrated in the heatmaps, the performance of $K \ge 1$ degrades rapidly if $h$ deviates from the optimal range (e.g., for $\alpha=1.2$, the error for $K=1$ spikes to $6.73$ at $h=0.03$). In contrast, the $K=0$ configuration remains stable across a wide spectrum of hyperparameters, consistently outperforming the unadjusted FULA baseline without requiring extensive tuning.
\end{itemize}

These findings suggest a practical \textit{training-inference asymmetry}: while the robust $K=0$ drift suffices for stable proposal generation in general settings, domain-specific fine-tuning of the approximation physics ($K, h$) allows practitioners to maximize sampling accuracy when ground-truth validation is feasible.

\begin{figure}[t!]
    \centering
    \begin{subfigure}[b]{0.32\textwidth}
        \centering
        \includegraphics[width=\textwidth]{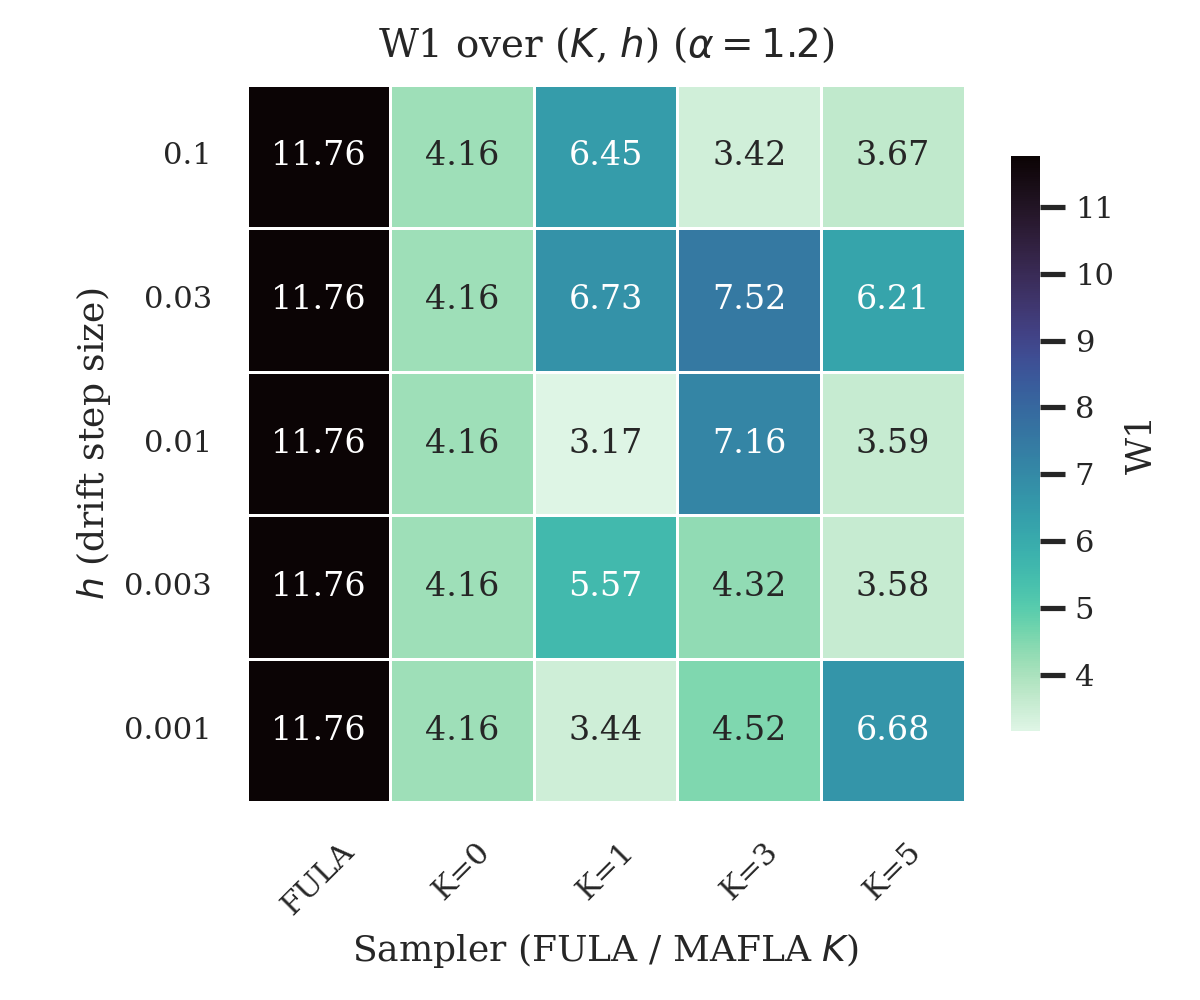}
        \caption{$W_1$ ($\alpha=1.2$)}
    \end{subfigure}
    \hfill
    \begin{subfigure}[b]{0.32\textwidth}
        \centering
        \includegraphics[width=\textwidth]{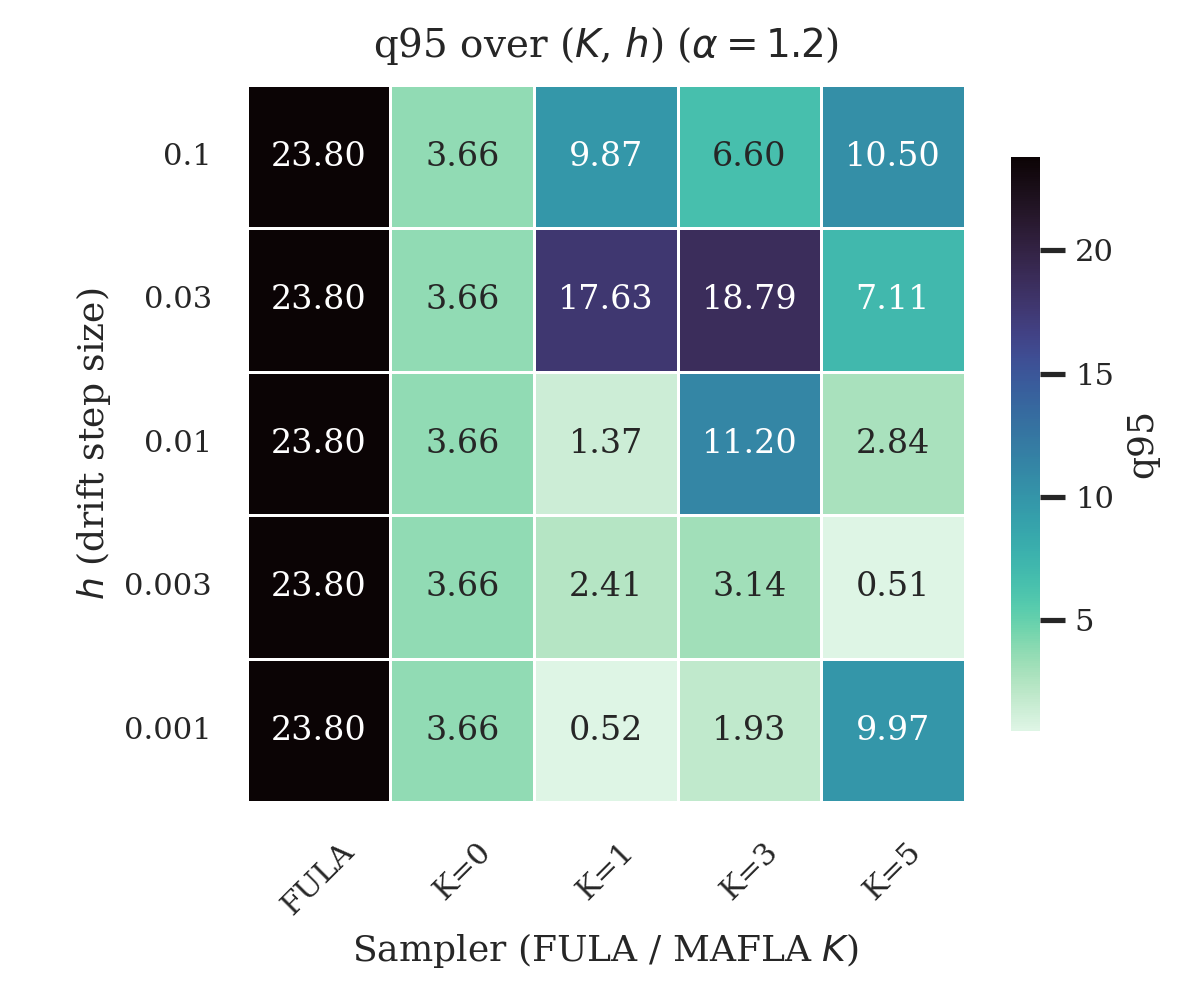}
        \caption{$q_{95}$ ($\alpha=1.2$)}
    \end{subfigure}
    \hfill
    \begin{subfigure}[b]{0.32\textwidth}
        \centering
        \includegraphics[width=\textwidth]{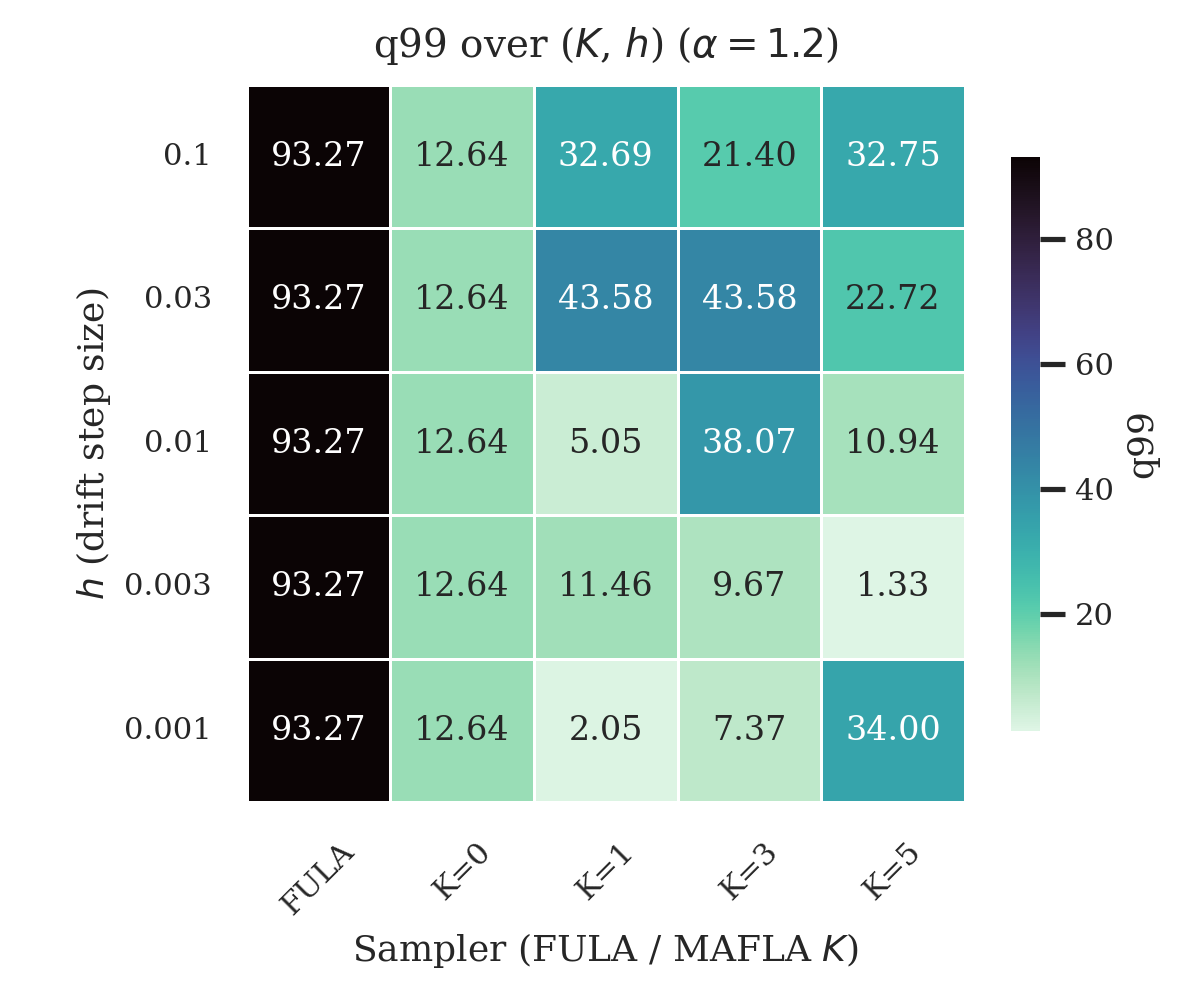}
        \caption{$q_{99}$ ($\alpha=1.2$)}
    \end{subfigure}
    
    \vspace{0.5em} 
    
    \begin{subfigure}[b]{0.32\textwidth}
        \centering
        \includegraphics[width=\textwidth]{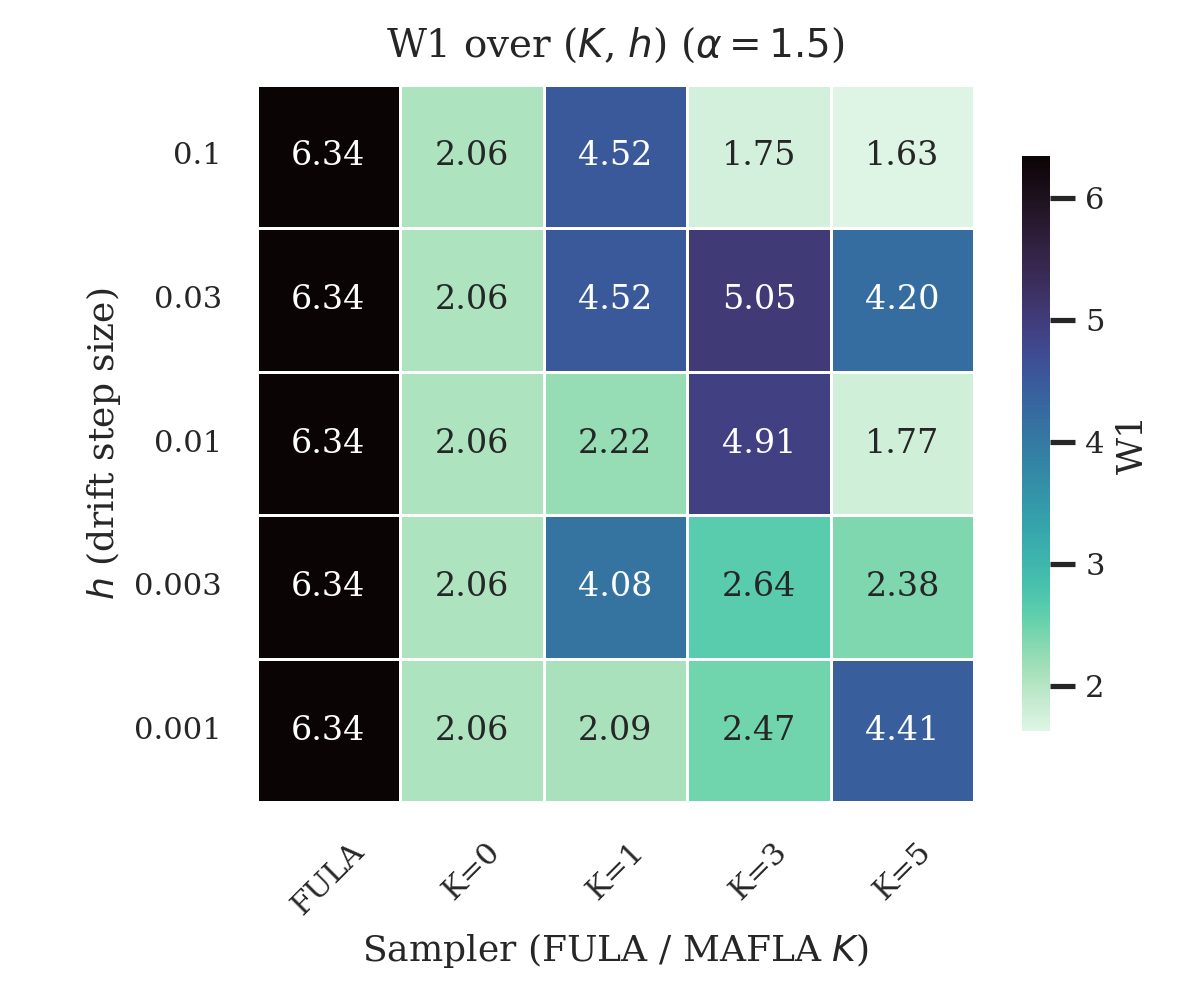}
        \caption{$W_1$ ($\alpha=1.5$)}
    \end{subfigure}
    \hfill
    \begin{subfigure}[b]{0.32\textwidth}
        \centering
        \includegraphics[width=\textwidth]{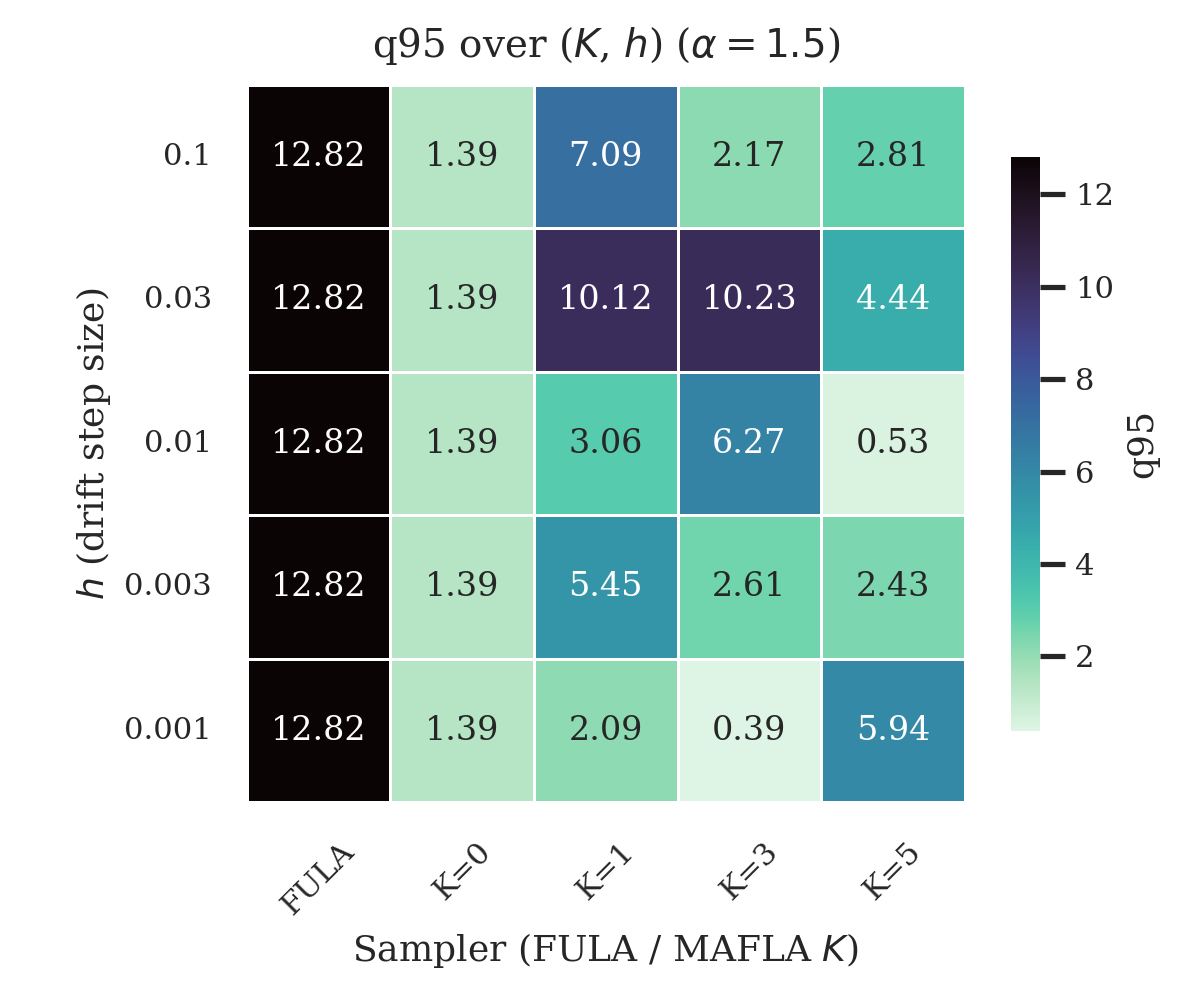}
        \caption{$q_{95}$ ($\alpha=1.5$)}
    \end{subfigure}
    \hfill
    \begin{subfigure}[b]{0.32\textwidth}
        \centering
        \includegraphics[width=\textwidth]{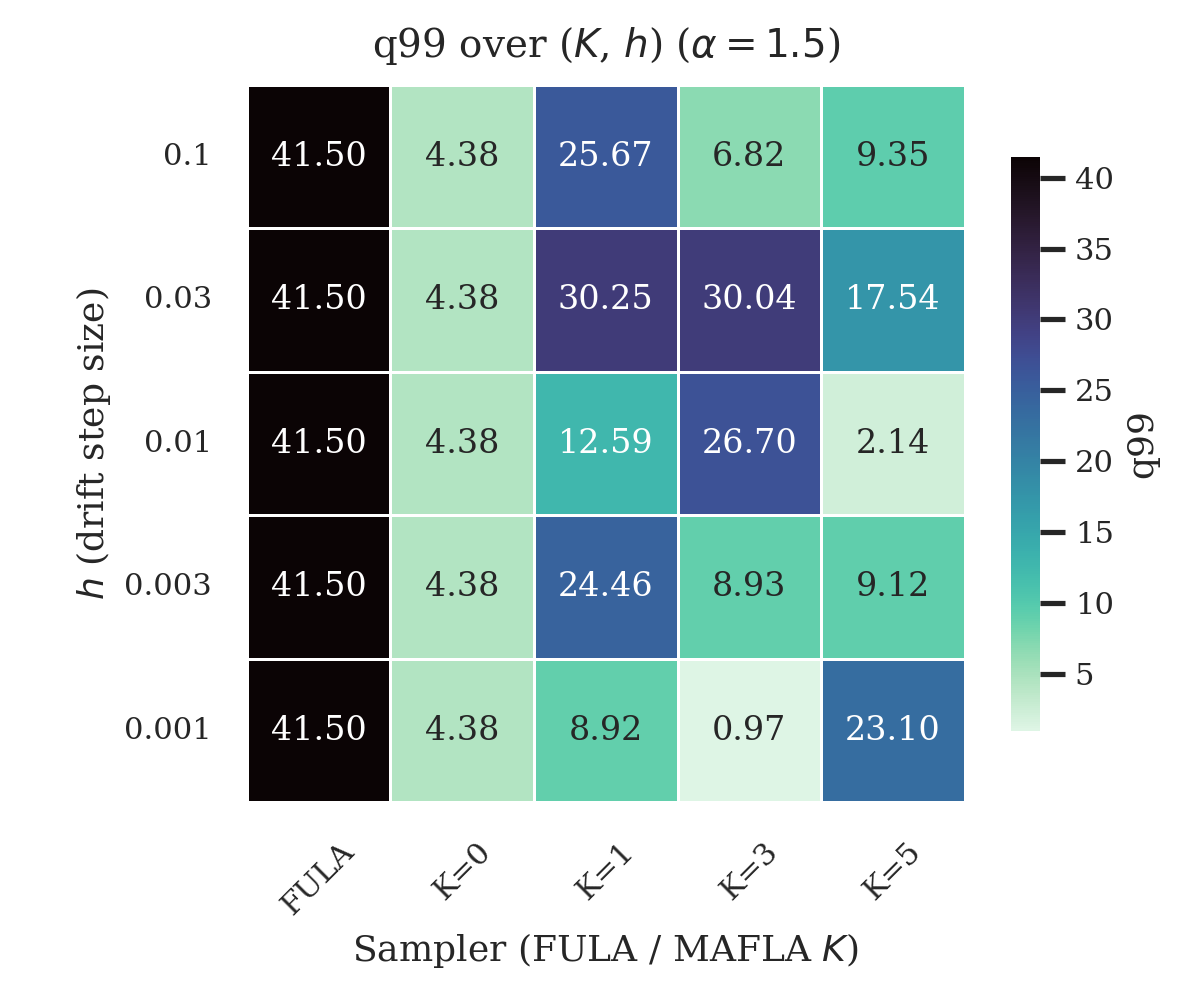}
        \caption{$q_{99}$ ($\alpha=1.5$)}
    \end{subfigure}
    
    \vspace{0.5em} 

    \begin{subfigure}[b]{0.32\textwidth}
        \centering
        \includegraphics[width=\textwidth]{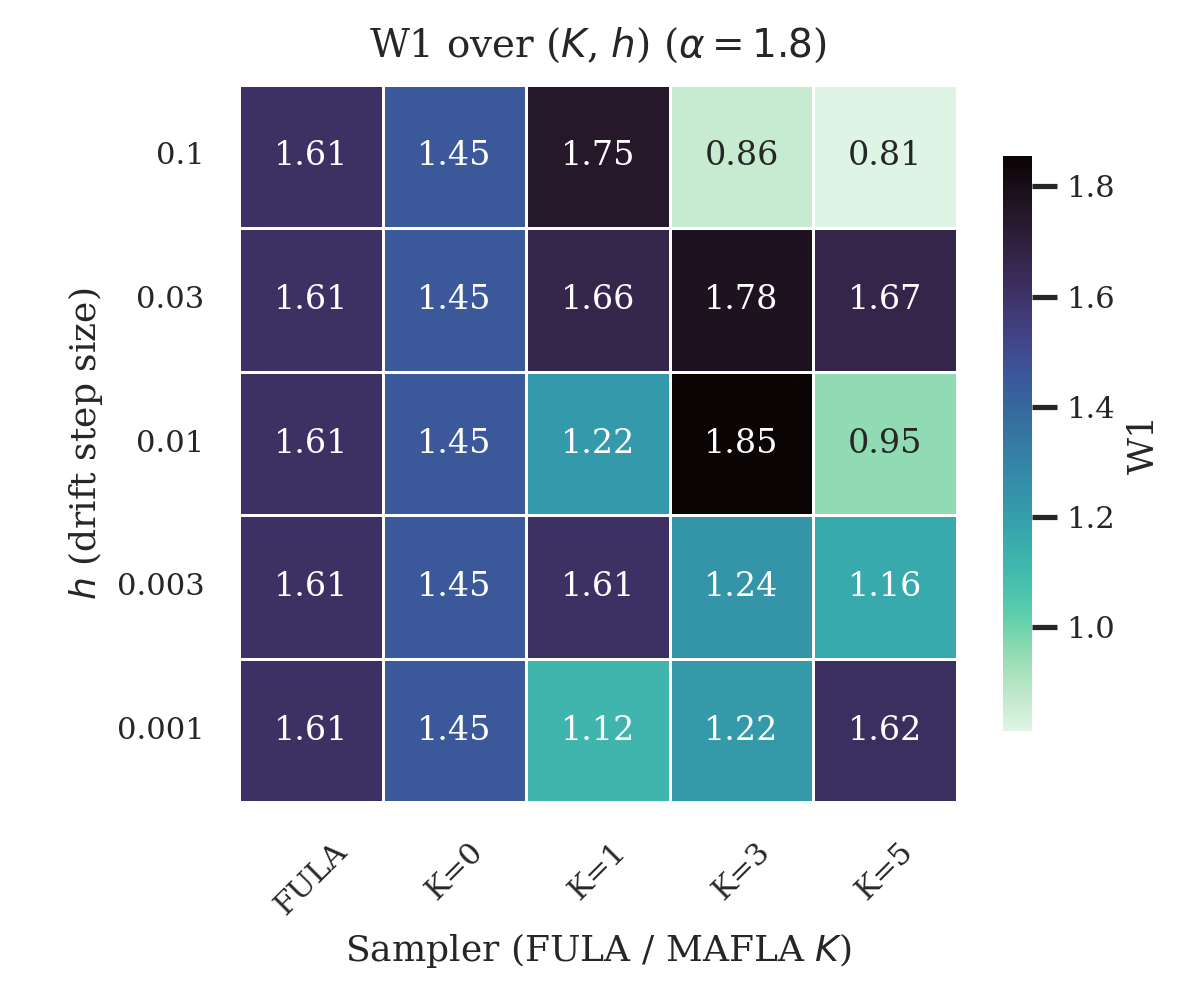}
        \caption{$W_1$ ($\alpha=1.8$)}
    \end{subfigure}
    \hfill
    \begin{subfigure}[b]{0.32\textwidth}
        \centering
        \includegraphics[width=\textwidth]{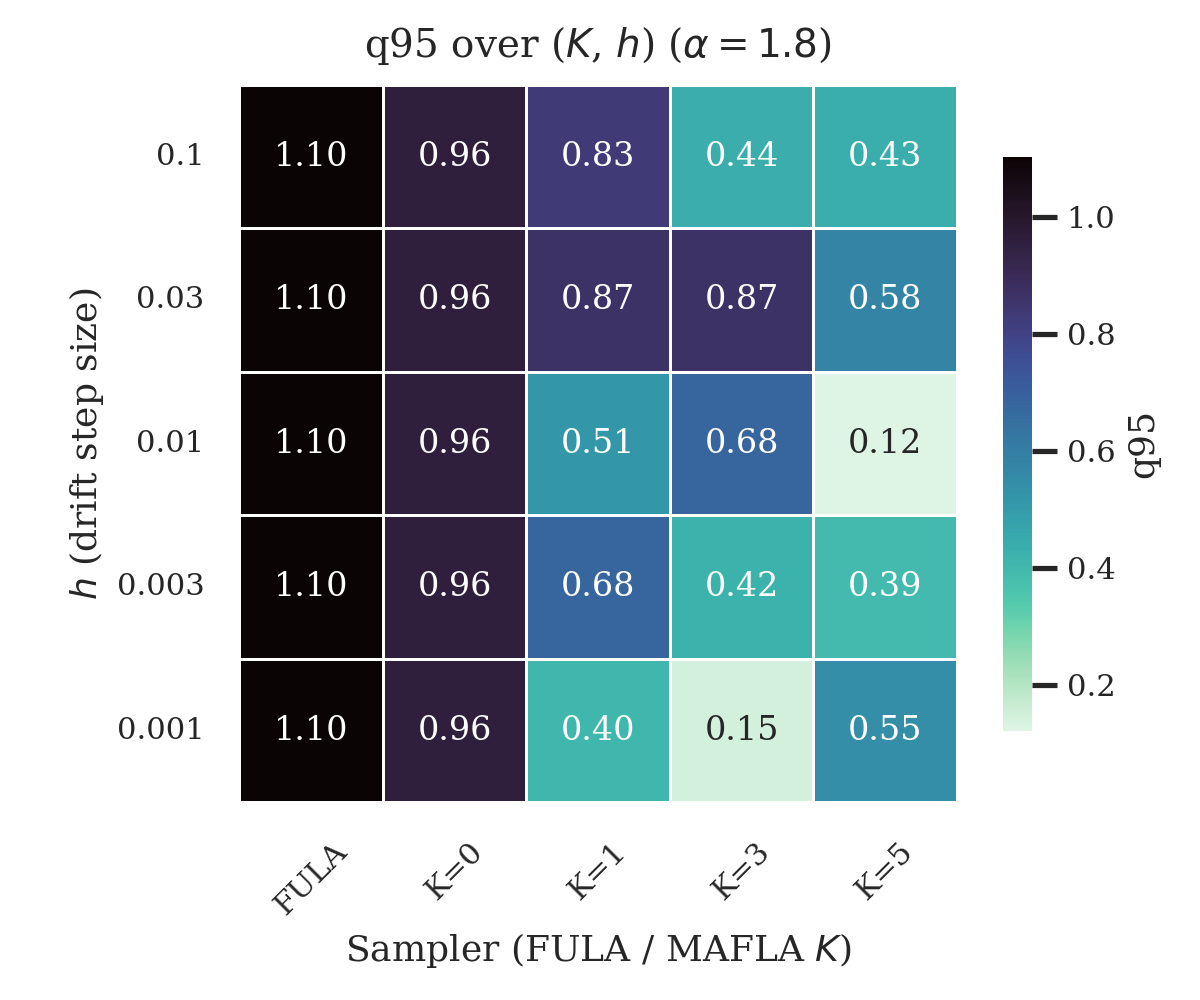}
        \caption{$q_{95}$ ($\alpha=1.8$)}
    \end{subfigure}
    \hfill
    \begin{subfigure}[b]{0.32\textwidth}
        \centering
        \includegraphics[width=\textwidth]{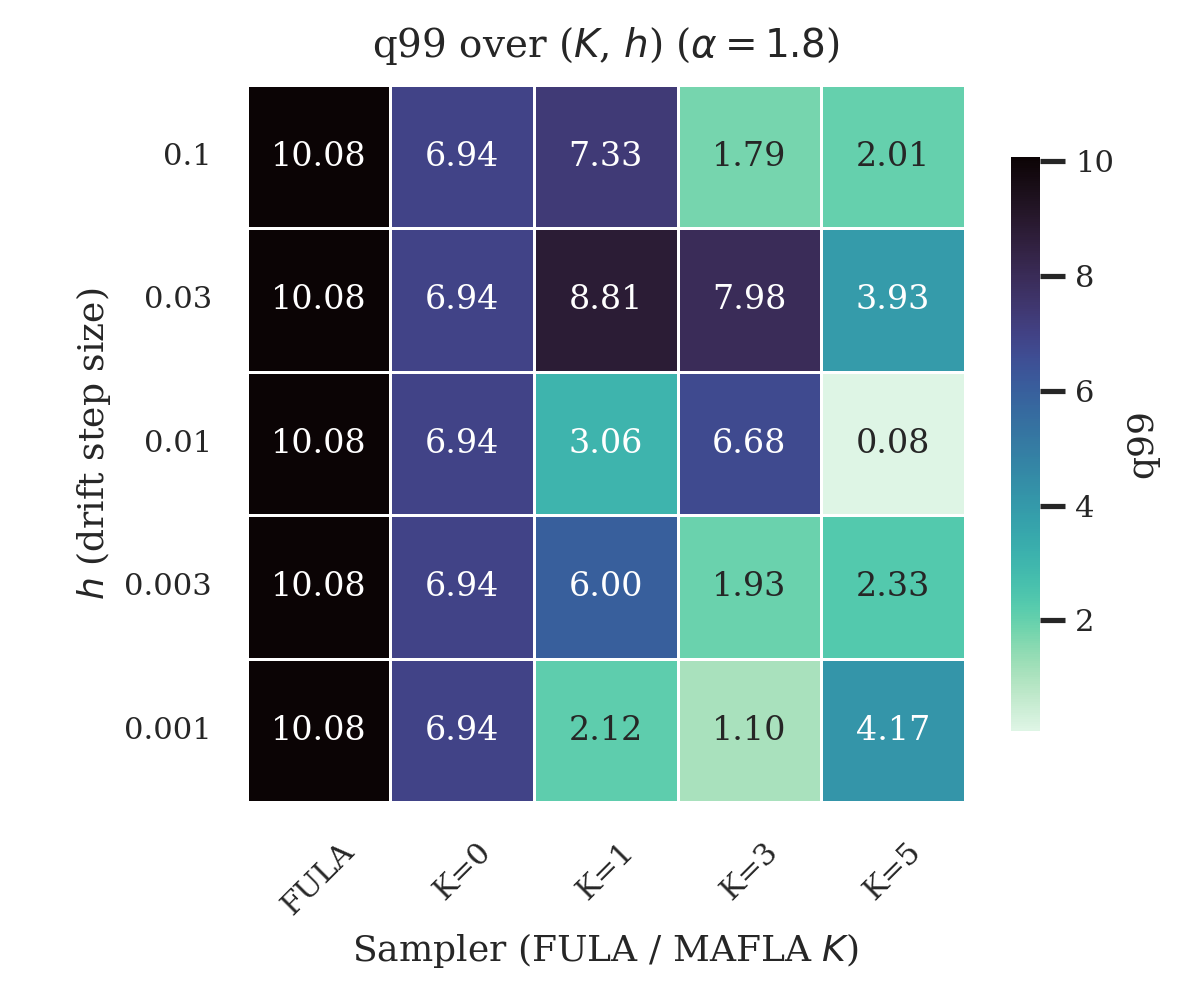}
        \caption{$q_{99}$ ($\alpha=1.8$)}
    \end{subfigure}

    \caption{\textbf{Effect of approximation order $K$ and drift step $h$ on sampling accuracy.} Rows correspond to stability indices $\alpha \in \{1.2, 1.5, 1.8\}$. Columns display the Wasserstein-1 distance ($W_1$) and tail percentile errors ($q_{95}, q_{99}$). While the baseline $K=0$ (second column in each heatmap) is consistently robust, fine-tuning the approximation order ($K \ge 1$) with an appropriate step size $h$ yields the lowest discrepancies.}
    \label{fig:heatmaps}
\end{figure}

\subsection{Ablation Study: Loss Regularization \texorpdfstring{$\lambda_\alpha$}{lambda\_alpha}}
\begin{figure}[h]
    \centering
    \begin{subfigure}[b]{0.32\textwidth}
        \centering
        \includegraphics[width=\textwidth]{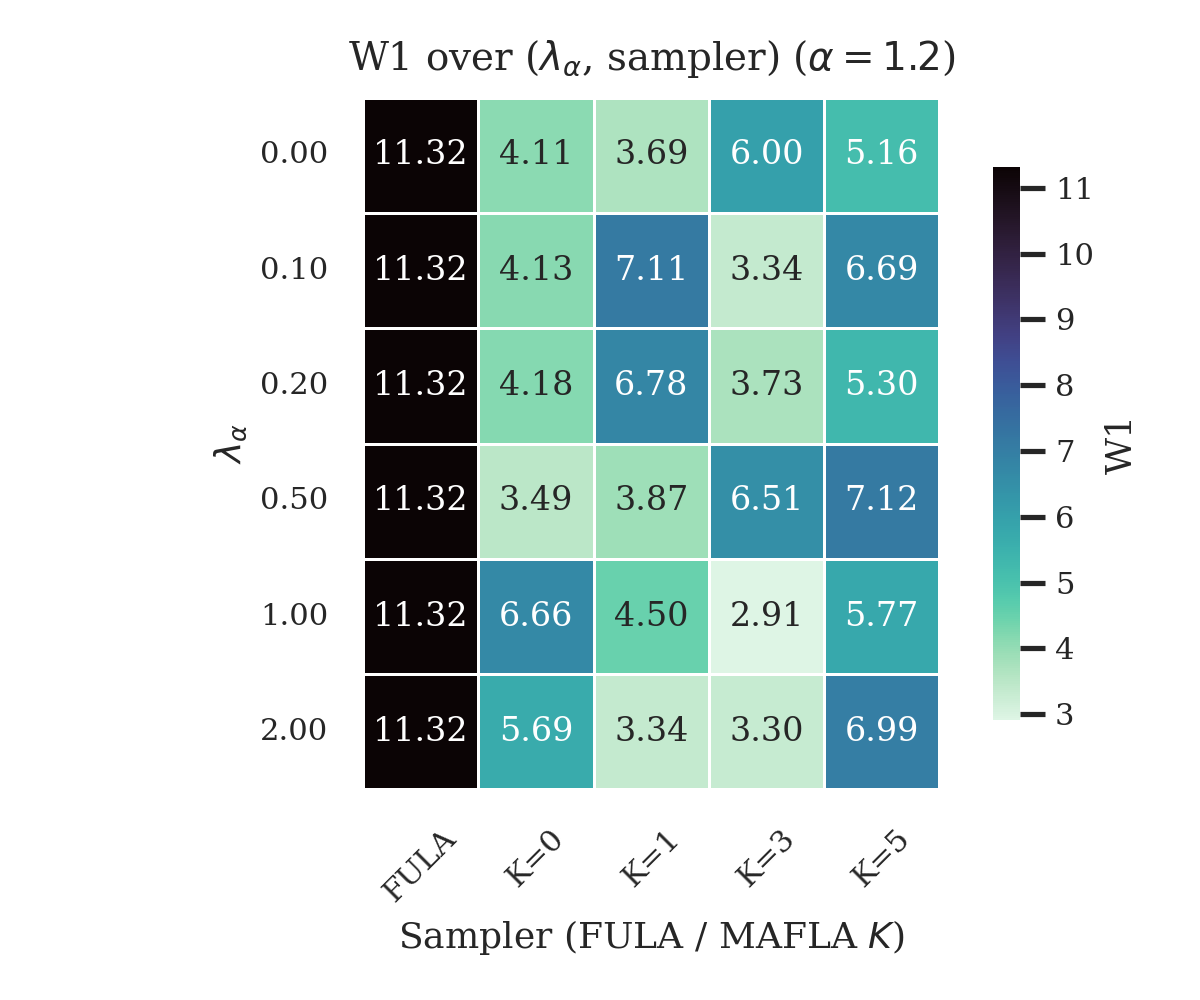}
        \caption{$W_1$ ($\alpha=1.2$)}
    \end{subfigure}
    \hfill
    \begin{subfigure}[b]{0.32\textwidth}
        \centering
        \includegraphics[width=\textwidth]{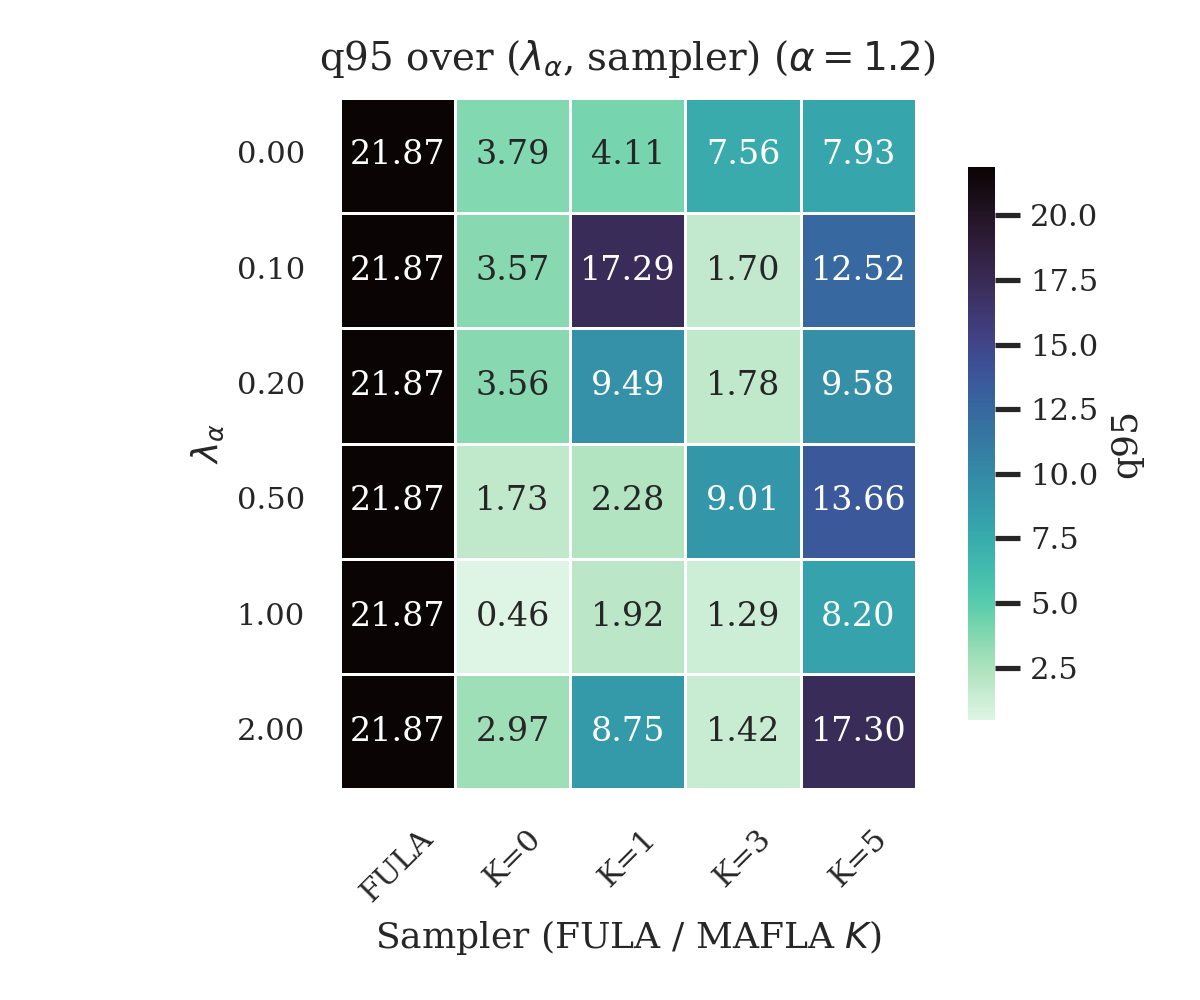}
        \caption{$q_{95}$ ($\alpha=1.2$)}
    \end{subfigure}
    \hfill
    \begin{subfigure}[b]{0.32\textwidth}
        \centering
        \includegraphics[width=\textwidth]{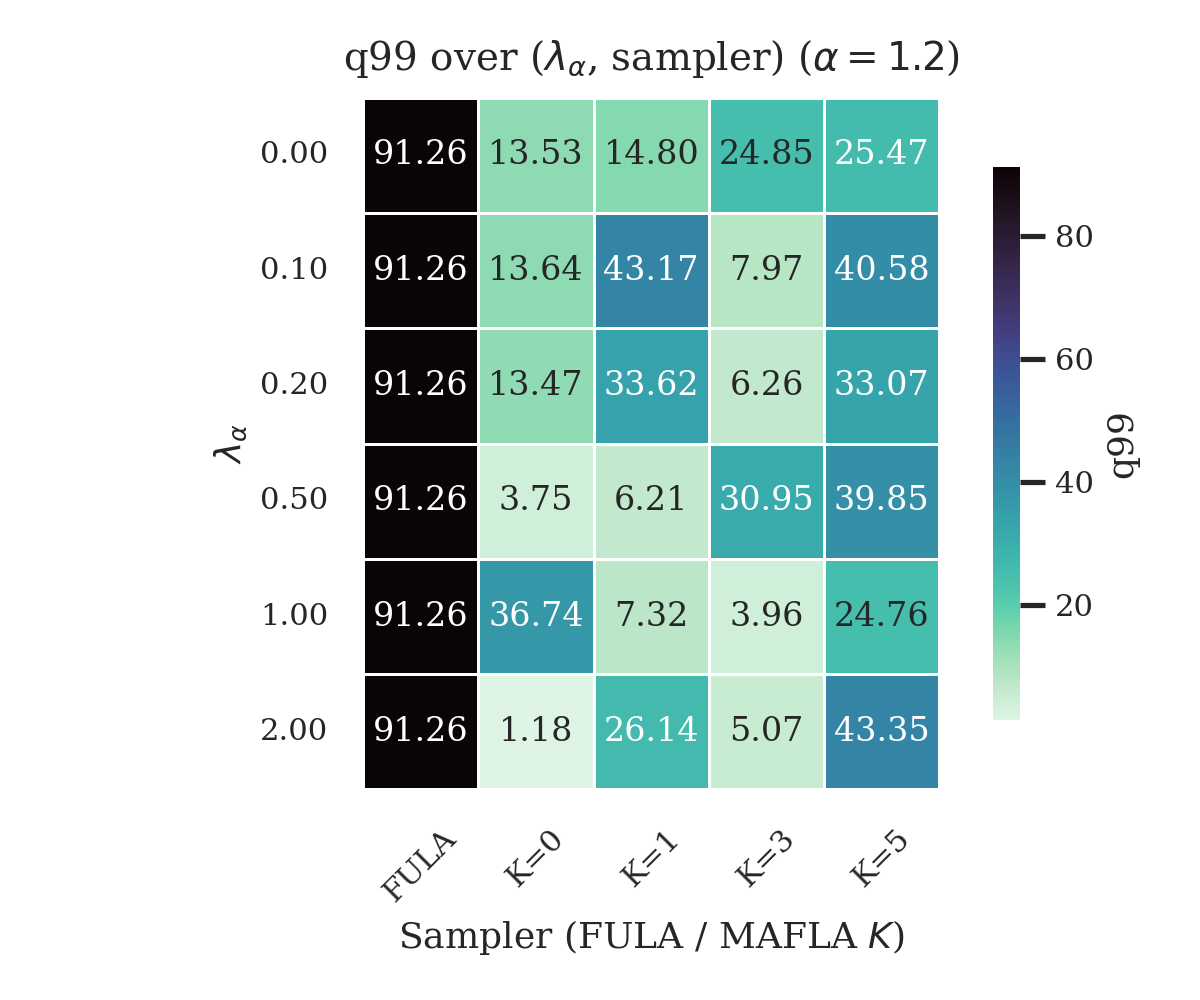}
        \caption{$q_{99}$ ($\alpha=1.2$)}
    \end{subfigure}
    
    \vspace{0.5em} 
    
    \begin{subfigure}[b]{0.32\textwidth}
        \centering
        \includegraphics[width=\textwidth]{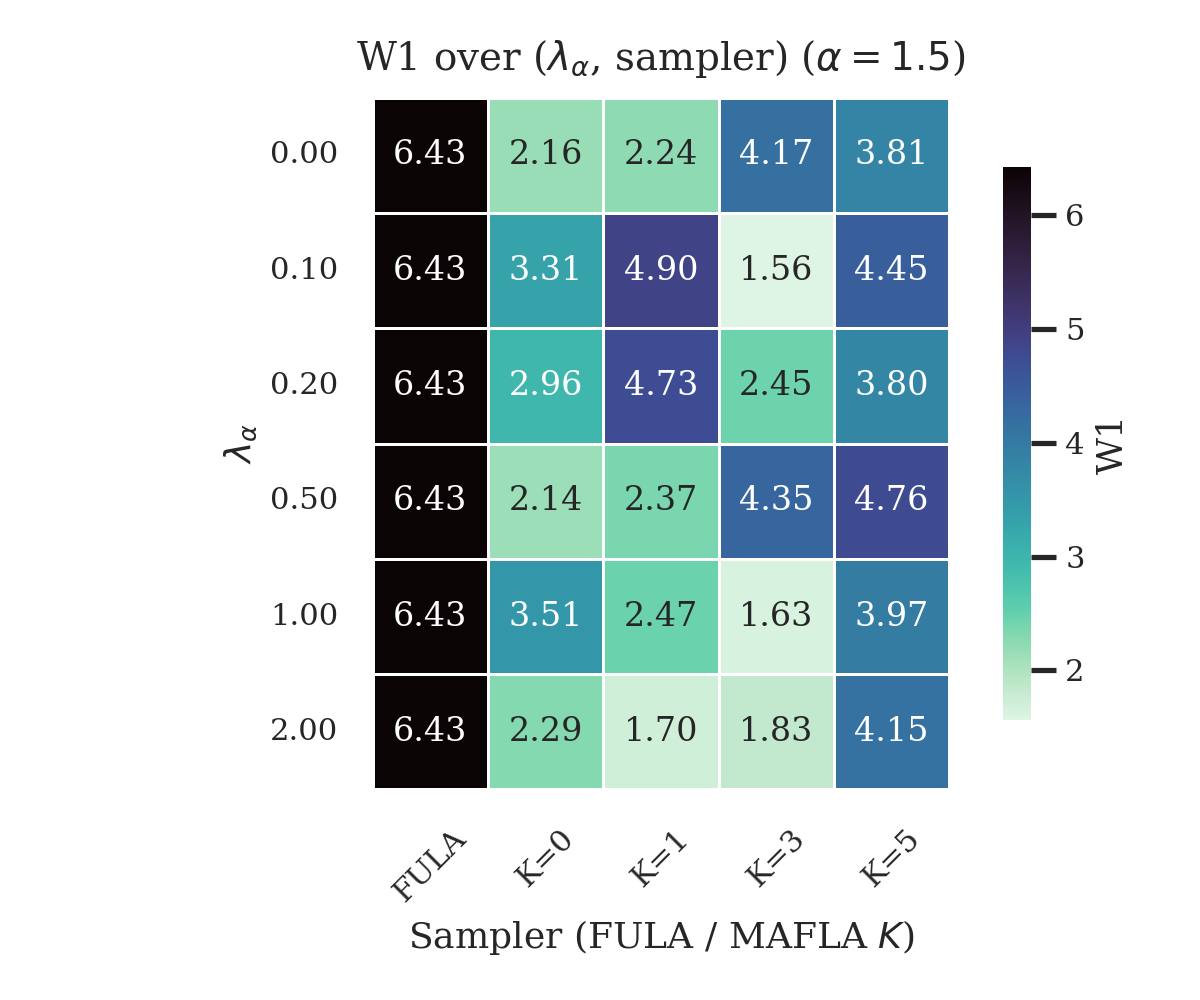}
        \caption{$W_1$ ($\alpha=1.5$)}
    \end{subfigure}
    \hfill
    \begin{subfigure}[b]{0.32\textwidth}
        \centering
        \includegraphics[width=\textwidth]{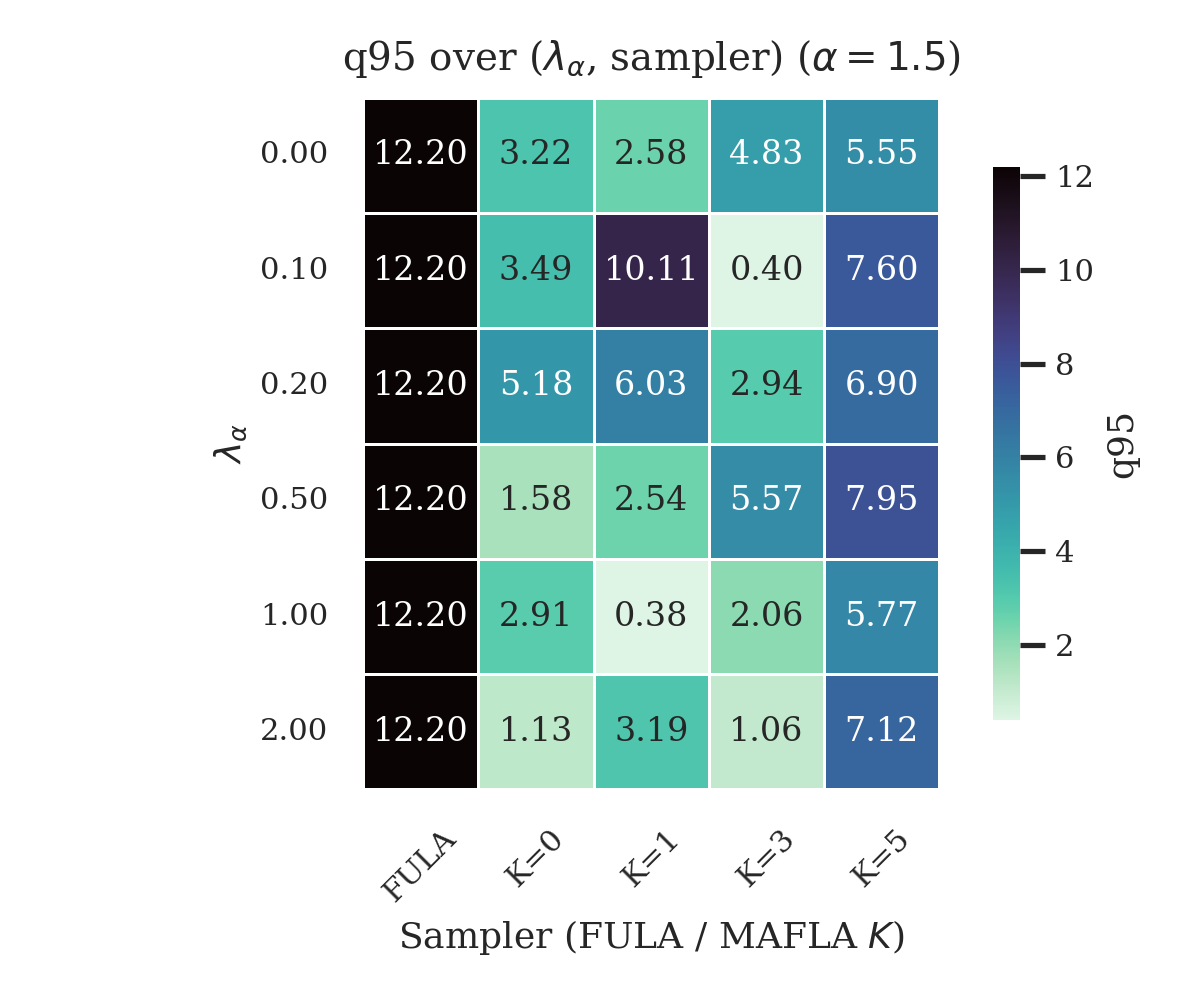}
        \caption{$q_{95}$ ($\alpha=1.5$)}
    \end{subfigure}
    \hfill
    \begin{subfigure}[b]{0.32\textwidth}
        \centering
        \includegraphics[width=\textwidth]{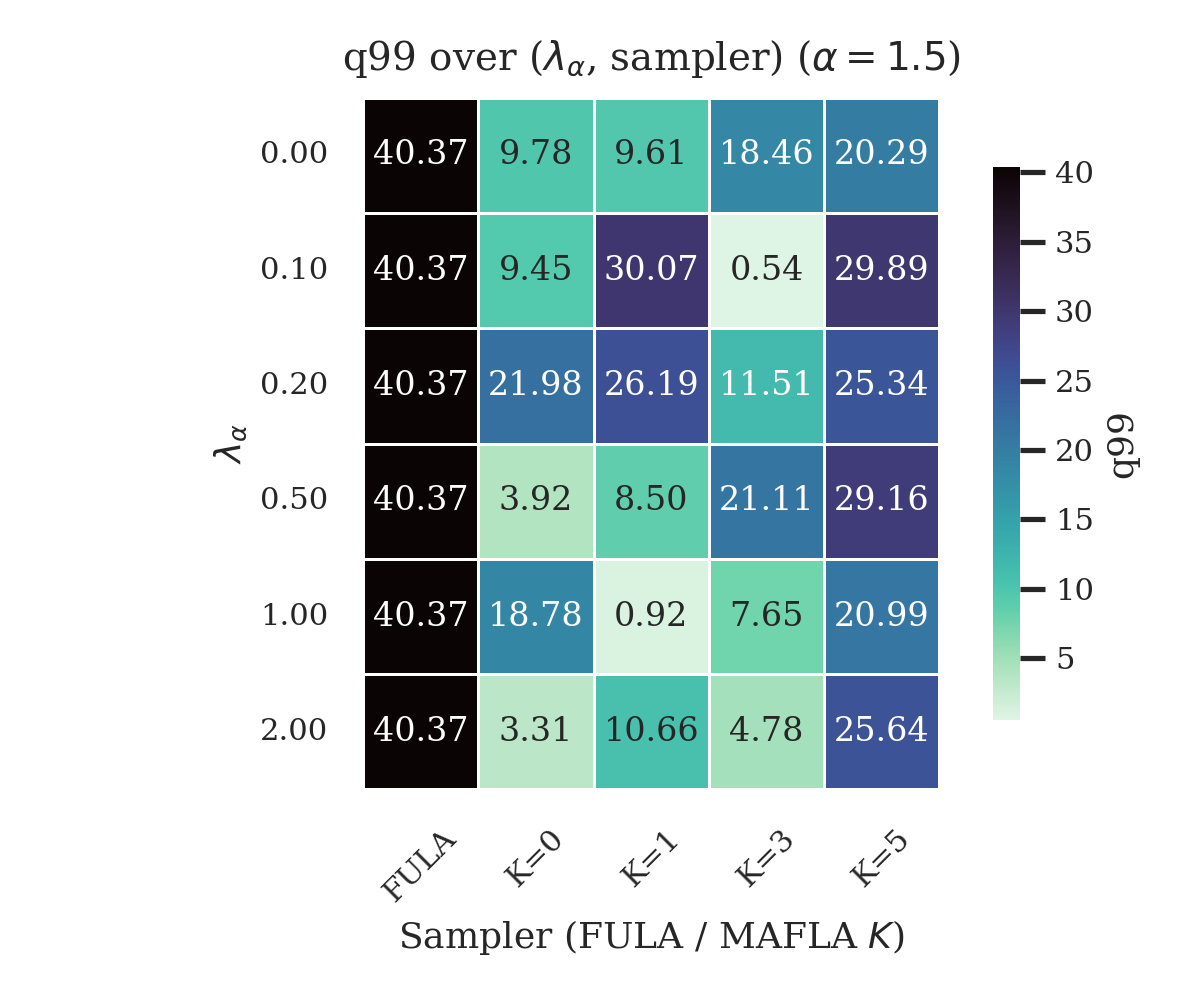}
        \caption{$q_{99}$ ($\alpha=1.5$)}
    \end{subfigure}
    
    \vspace{0.5em}

    \begin{subfigure}[b]{0.32\textwidth}
        \centering
        \includegraphics[width=\textwidth]{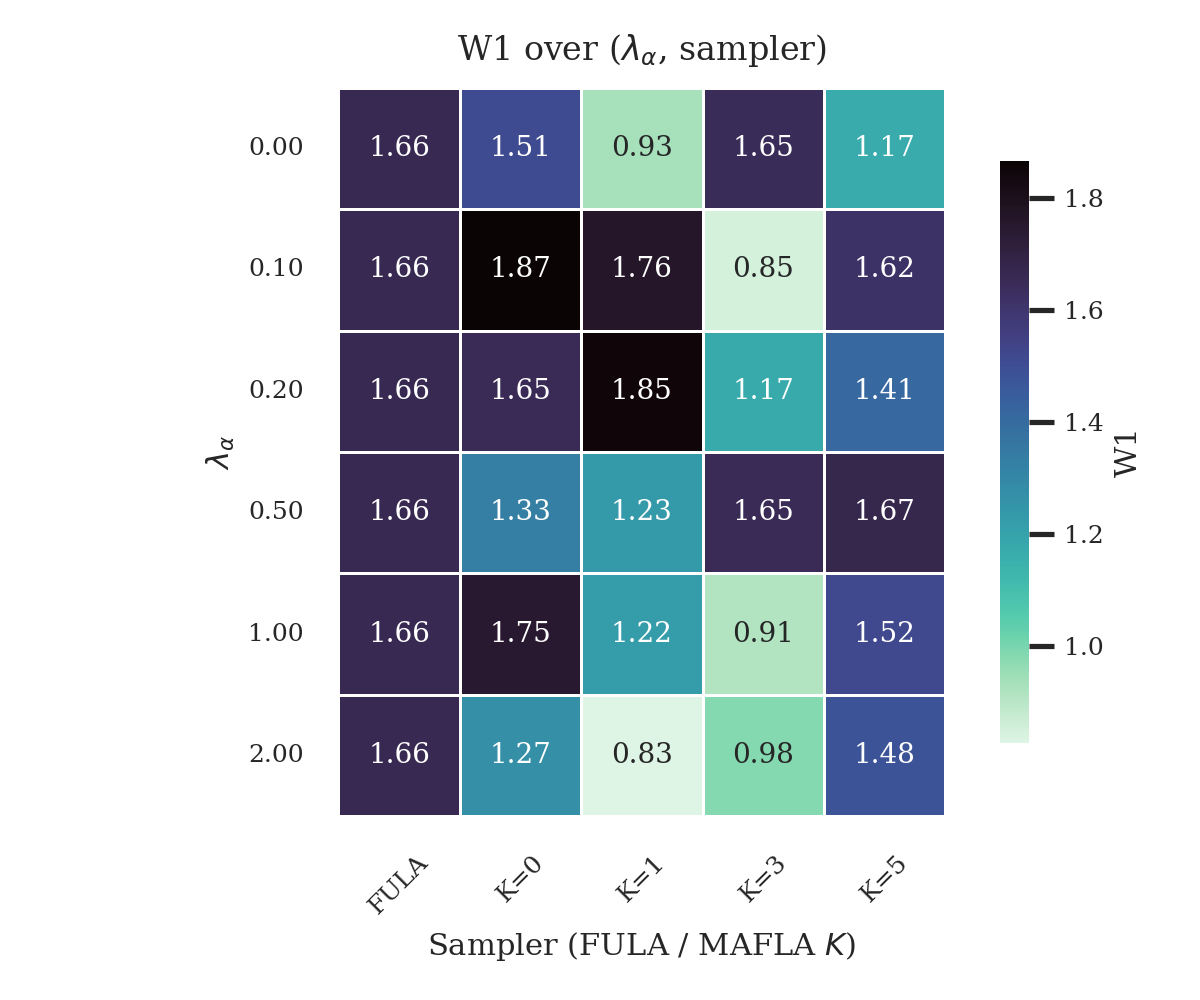}
        \caption{$W_1$ ($\alpha=1.8$)}
    \end{subfigure}
    \hfill
    \begin{subfigure}[b]{0.32\textwidth}
        \centering
        \includegraphics[width=\textwidth]{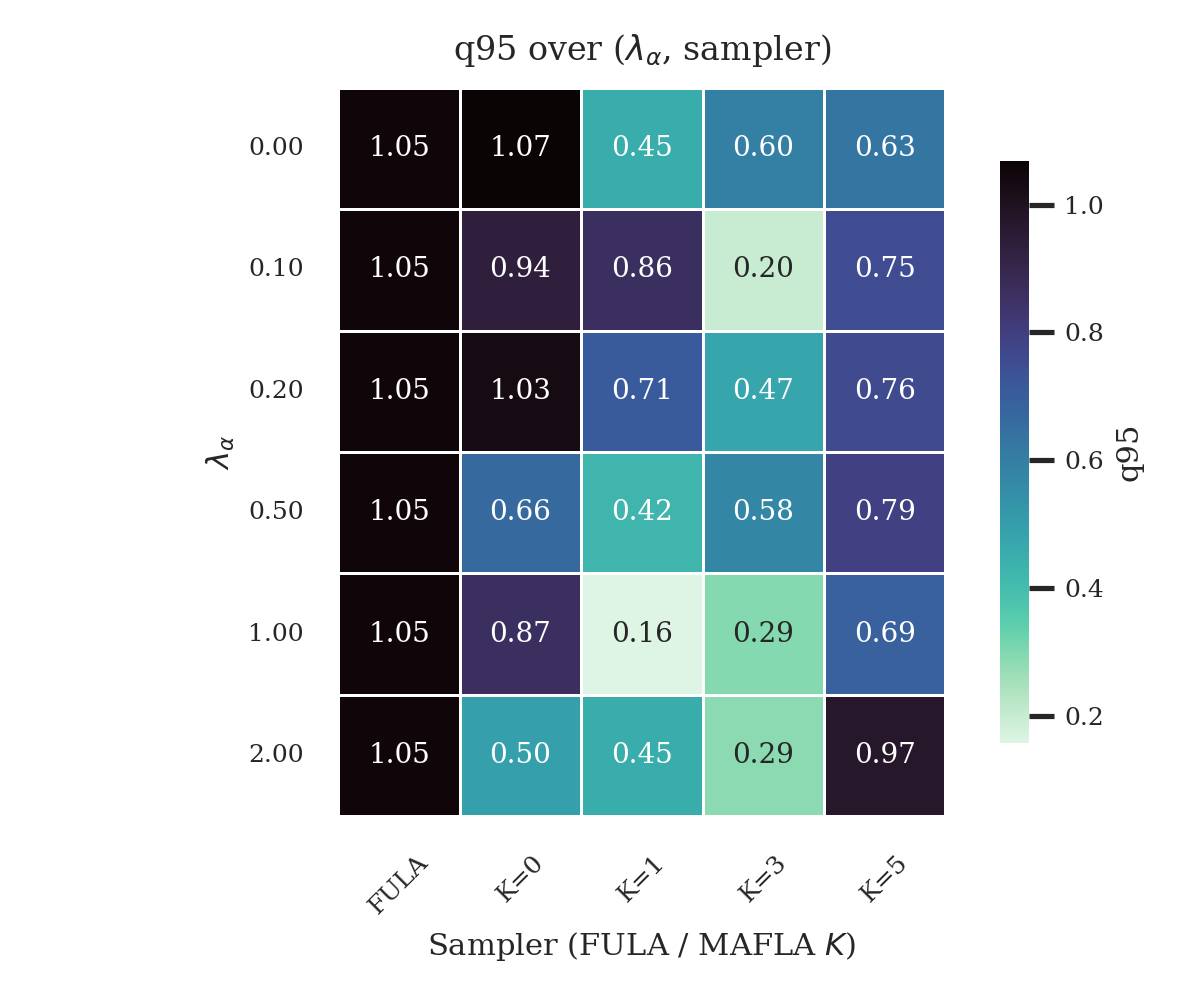}
        \caption{$q_{95}$ ($\alpha=1.8$)}
    \end{subfigure}
    \hfill
    \begin{subfigure}[b]{0.32\textwidth}
        \centering
        \includegraphics[width=\textwidth]{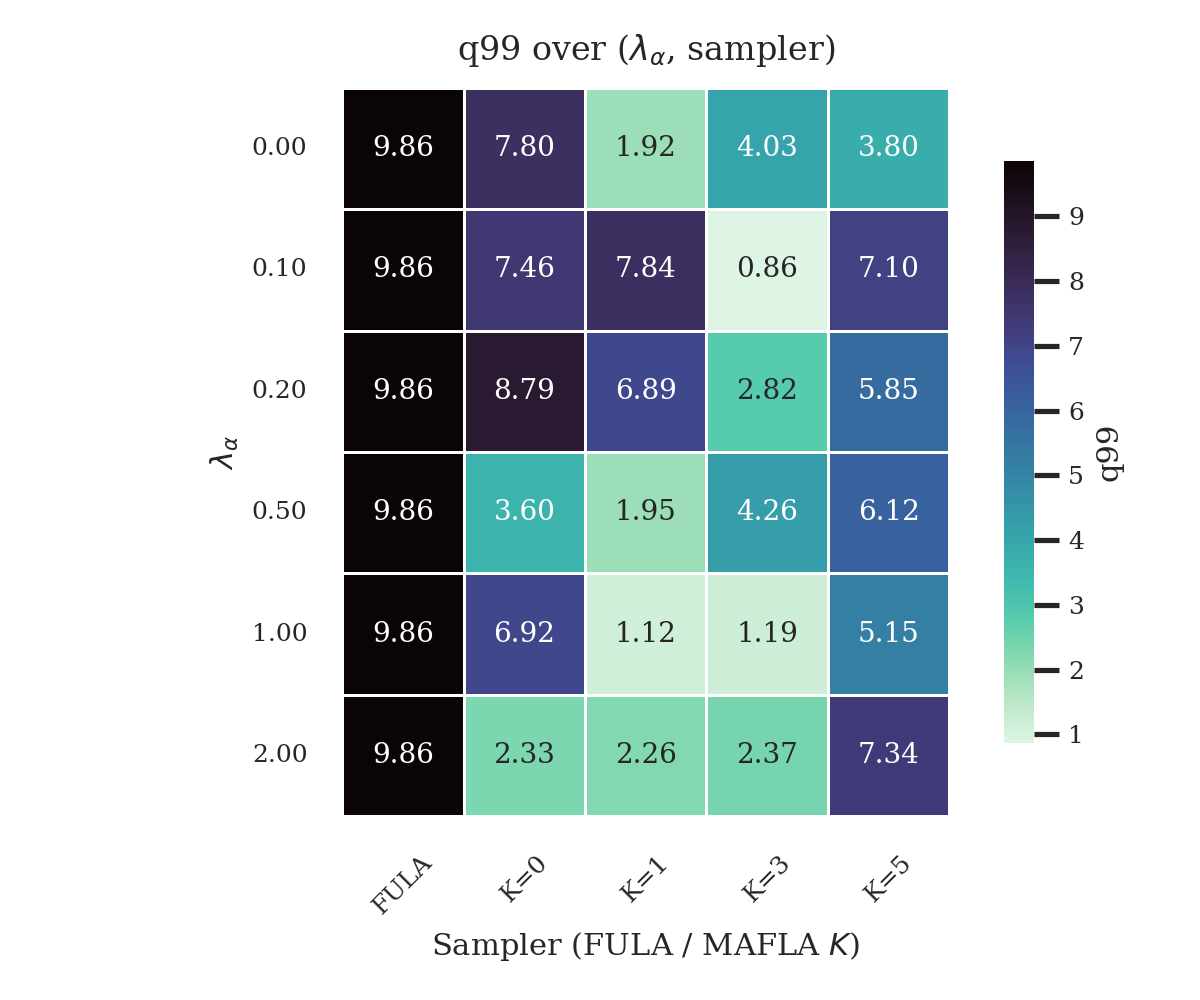}
        \caption{$q_{99}$ ($\alpha=1.8$)}
    \end{subfigure}

    \caption{Effect of loss regularization $\lambda_\alpha$ and approximation order $K$. The heatmaps show performance at fixed drift step $h=0.01$. The inclusion of the $\alpha$-norm in the Score Balance Matching objective ($\lambda_\alpha > 0$) provides significant robustness, particularly in reducing tail errors ($q_{99}$) for heavy-tailed distributions ($\alpha=1.2, 1.5$).}
    \label{fig:heatmaps_lambda}
\end{figure}

\label{sec:ablation_lambda}

To address the heavy-tailed nature of the score residuals, we proposed a hybrid Score Balance Matching objective defined as $\mathcal{L} = \mathcal{L}_2 + \lambda_\alpha \mathcal{L}_\alpha$. In this study, we investigated the impact of the regularization weight $\lambda_\alpha \in [0, 2]$ alongside the approximation order $K$, fixing the drift step size at $h=10^{-2}$.

The results, visualized in Figure~\ref{fig:heatmaps_lambda}, demonstrate that incorporating the $\alpha$-norm ($\lambda_\alpha > 0$) significantly improves sampling stability, particularly for tail statistics. 
\begin{itemize}
    \item \textbf{Tail Stability:} For heavy-tailed targets ($\alpha=1.2$), a purely $L^2$-based loss ($\lambda_\alpha=0$) yields a 99th-percentile error of $14.80$ (using $K=1$). In contrast, setting $\lambda_\alpha=2.0$ reduces this error to $1.18$ (using $K=0$), effectively constraining the variance of the learned acceptance probability in the tails.
    \item \textbf{Consistent Enhancement:} Overall, higher regularization weights ($\lambda_\alpha \in [0.5, 2.0]$) consistently enhance the sampler's ability to fit extreme quantiles without compromising the Wasserstein-1 distance.
\end{itemize}

These findings indicate that the mixed-norm objective is crucial for learning robust acceptance functions when the proposal distribution exhibits infinite variance.

\section{Case Study of Combinatorial Optimization Problems}\label{sec:codetail}
\subsection{MaxCut}\label{app:maxcut_details}

\paragraph{Combinatorial formulation.}
The MaxCut problem on an undirected graph $G=(V,E)$ with $|V|=N$ can be written as
\begin{equation*}
\min_{x\in\{-1,1\}^N} \sum_{i,j=1}^N w_{ij} x_i x_j
= x^\top W x ,
\end{equation*}
where $W=(w_{ij})_{1\le i,j\le N}$ is a symmetric weight matrix.
In our experiments, we set $W$ to be the adjacency matrix $A$ of the graph,
defined by
\[
A_{ij}=\mathbf{1}_{\{(i,j)\in E\}} .
\]
Under this choice, minimizing $x^\top W x$ is equivalent (up to an additive
constant) to maximizing the number of edges whose endpoints lie in different
partitions, i.e., the MaxCut objective.

\paragraph{Continuous relaxation and energy model.}
Directly sampling over the discrete domain $\{-1,1\}^N$ is infeasible for
gradient-based methods. We therefore introduce a continuous latent variable
$u\in\mathbb{R}^N$ and define a smooth relaxation through the elementwise
transformation
\[
y=\tanh(u)\in(-1,1)^N .
\]
Using this relaxed variable, we define the energy function $E:\mathbb{R}^N\to\mathbb{R}$ as
\begin{equation*}
E(u)
=\frac{1}{2}\sum_{i,j=1}^N w_{ij}\tanh(u_i)\tanh(u_j)
=\frac{1}{2}\tanh(u)^\top W\tanh(u).
\end{equation*}
Equivalently, in terms of $y$, we have $E(u)=E(y)=\frac{1}{2}y^\top Wy$.
This relaxation preserves the quadratic structure of the original objective
while yielding a smooth, differentiable energy landscape.

We define a target density over $u$ by
\begin{equation*}
\pi(u)\propto \exp\!\left(-\frac{E(u)}{\eta}\right),\quad u\in\mathbb{R}^N,
\end{equation*}
where $\eta>0$ is a temperature parameter. Smaller values of $\eta$ concentrate probability mass around low-energy regions,
corresponding to high-quality cuts, while larger values encourage broader
exploration.

\paragraph{Score function.}
The score function $\nabla_u \log \pi(u)$ provides the drift term for
Langevin-type sampling algorithms. By the quadratic structure of $E(y)$, we have
\[
\nabla_y E(y)=Wy.
\]
Applying the chain rule yields
\begin{equation*}
\nabla_u E(u)
=\bigl(1-\tanh(u)^2\bigr)\odot W\tanh(u),
\end{equation*}
where $\odot$ denotes elementwise multiplication. Consequently, the score function takes the explicit form
\begin{equation*}
\nabla_u \log \pi(u)
= -\frac{1}{\eta}\nabla_u E(u)
= -\frac{1}{\eta}\bigl(1-\tanh(u)^2\bigr)\odot W\tanh(u).
\end{equation*}
This closed-form (oracle) score is used by ULA, MALA, and their fractional
counterparts to guide sampling toward regions corresponding to large cuts.

\paragraph{Discretization strategy.}
After sampling a continuous state $u$, we recover a discrete cut by applying the
sign map
\begin{equation}
x=\mathrm{sgn}\!\left(\tanh(u)\right),
\qquad
\mathrm{sgn}(a)=\mathbf{1}_{\{a\ge 0\}}-\mathbf{1}_{\{a<0\}}, \quad a\in\mathbb{R}.
\end{equation}
This operation projects the relaxed variable back onto $\{-1,1\}^N$ and yields a valid partition of the graph. The resulting discrete solution is then evaluated using the original MaxCut objective.

\subsection{Vertex Cover}\label{app:vc_details}

\paragraph{Combinatorial formulation.}
Given an undirected graph $G=(V,E)$ with $|V|=N$, the minimum vertex cover problem seeks a binary assignment $x\in\{0,1\}^N$ minimizing
\begin{equation*}
\min_{x\in\{0,1\}^N} \sum_{i=1}^N x_i
\quad \text{subject to} \quad
x_i + x_j \ge 1 \quad \forall (i,j)\in E,
\end{equation*}
so that every edge is incident to at least one selected vertex.
This problem is NP-hard and serves as a canonical example of constrained
combinatorial optimization.

\paragraph{Continuous relaxation and energy model.}
To enable gradient-based sampling, we introduce a continuous latent variable
$u\in\mathbb{R}^N$ and relax the binary variables via the sigmoid transformation
\[
p=\sigma(u)\in(0,1)^N,
\qquad
\sigma(a)=\frac{1}{1+e^{-a}} .
\]
Here, $p_i$ may be interpreted as a soft selection indicator for vertex $i$.

We define a relaxed energy function consisting of two components: a soft
cardinality term and a penalty enforcing edge coverage,
\begin{equation*}
E(u)
=\sum_{i=1}^N y_i
+\lambda \sum_{(i,j)\in E} \bigl(1-y_i\bigr)\bigl(1-y_j\bigr)=\mathbf{1}^\top p+\frac{\lambda}{2}(\mathbf{1}-p)^\top A(\mathbf{1}-p),
\end{equation*}
where $A=(\mathbf{1}_{(i,j)\in E})_{1\leq i, j\leq N}$ and $\lambda>0$ controls the strength of the constraint enforcement.
This relaxation penalizes uncovered edges while encouraging sparse vertex
selections, and reduces to the original objective in the binary limit.

As in the MaxCut case, we define a target density
\begin{equation*}
\pi(u)\propto \exp\!\left(-\frac{E(u)}{\eta}\right),
\end{equation*}
where $\eta>0$ is a temperature parameter controlling the exploration--exploitation
trade-off.

\paragraph{Score function.}
The score function governing Langevin-type dynamics is given by
$$\nabla_u \log \pi(u) = -\frac{1}{\eta}\nabla_u E(u).$$
Using the chain rule and the derivative
$\sigma'(u)=\sigma(u)\bigl(1-\sigma(u)\bigr)$, we obtain
\begin{equation*}
\nabla_u E(u)=\sigma'(u)\odot\left[\mathbf{1}-\lambda A(\mathbf{1}-p)\right],
\end{equation*}
where the indicator term selects edges that remain uncovered under the relaxed assignment. This explicit (oracle) score is used by ULA, MALA, and their fractional variants to guide sampling toward low-energy, approximately feasible vertex covers.

\paragraph{Discretization and greedy decoding.}
After sampling the continuous variable $u$, we obtain an initial discrete vertex selection by thresholding:
\begin{equation*}
x_i=\mathbf{1}_{\{u_i>0\}}, \qquad i=1,\dots,N .
\end{equation*}
This produces a sparse candidate cover but may leave a subset of edges uncovered, since feasibility is not strictly enforced under the continuous relaxation.

To obtain a valid vertex cover, we apply a greedy decoding procedure.
Specifically, while there exists an uncovered edge $(i,j)\in E$ with
$x_i=x_j=0$, we add to the cover the endpoint with larger relaxed activation,
i.e.,
\[
x_{k}\leftarrow 1, \qquad
k=\arg\max_{\ell\in\{i,j\}} p_\ell ,
\]
where $p=\sigma(u)$.
This process is repeated until all edges are covered. The resulting discrete solution is guaranteed to be feasible and is evaluated
using the original vertex cover objective.

\begin{table*}[h]
\centering
\caption{Vertex Cover results on Erd\H{o}s--R\'enyi (ER) and Barab\'asi--Albert (BA) graphs.
Reported are mean\,$\pm$\,std energy, mean cover size (smaller is better), best cover found,
and uncovered ratio.}
\label{tab:vertex_cover}

\small
\begin{subtable}[t]{\textwidth}
\centering
\caption{Erd\H{o}s--R\'enyi graphs, with $\vert E\vert=2.5N$}
\begin{tabular}{llcccc}
\toprule
Graph & Sampler & Energy ($\mu\pm\sigma$) & Cover$^{\downarrow}$ ($\mu\pm\sigma$) & Best$^{\downarrow}$ & Uncov.$^{\downarrow}$ ($\mu\pm\sigma$) \\
\midrule
\multirow{4}{*}{ER64}
 & ULA   & $0.805\pm0.034$ & $49.88\pm2.10$ & 43 & $0.239\pm0.071$ \\
 & FULA  & $0.682\pm0.031$ & $46.17\pm1.96$ & 41 & $0.144\pm0.042$ \\
 & MALA  & $0.806\pm0.029$ & $49.88\pm1.95$ & 44 & $0.240\pm0.067$ \\
 & MAFLA & $0.657\pm0.028$ & $\mathbf{45.02\pm2.10}$ & $\mathbf{39}$ & $\mathbf{0.127\pm0.035}$ \\
\midrule
\multirow{4}{*}{ER256}
 & ULA   & $0.810\pm0.018$ & $198.66\pm4.42$ & 183 & $0.248\pm0.035$ \\
 & FULA  & $0.709\pm0.018$ & $187.14\pm4.53$ & 173 & $0.160\pm0.025$ \\
 & MALA  & $0.812\pm0.016$ & $199.03\pm4.35$ & 188 & $0.247\pm0.035$ \\
 & MAFLA & $0.685\pm0.014$ & $\mathbf{181.97\pm4.27}$ & $\mathbf{169}$ & $\mathbf{0.139\pm0.020}$ \\
\midrule
\multirow{4}{*}{ER512}
 & ULA   & $0.811\pm0.013$ & $398.60\pm6.19$ & 379 & $0.248\pm0.026$ \\
 & FULA  & $0.743\pm0.013$ & $385.83\pm6.54$ & 364 & $0.179\pm0.019$ \\
 & MALA  & $0.812\pm0.012$ & $398.93\pm6.41$ & 378 & $0.249\pm0.026$ \\
 & MAFLA & $0.727\pm0.011$ & $\mathbf{379.06\pm6.19}$ & $\mathbf{358}$ & $\mathbf{0.158\pm0.016}$ \\
\midrule
\multirow{4}{*}{ER1024}
 & ULA   & $0.812\pm0.009$ & $796.43\pm8.34$ & 771 & $0.249\pm0.018$ \\
 & FULA  & $0.766\pm0.009$ & $780.81\pm8.49$ & 753 & $0.196\pm0.015$ \\
 & MALA  & $0.812\pm0.008$ & $796.77\pm8.68$ & 768 & $0.249\pm0.017$ \\
 & MAFLA & $0.759\pm0.008$ & $\mathbf{775.94\pm8.82}$ & $\mathbf{747}$ & $\mathbf{0.184\pm0.013}$ \\
\bottomrule
\end{tabular}
\end{subtable}

\vspace{1em}

\begin{subtable}[t]{\textwidth}
\centering
\caption{Barab\'asi--Albert graphs, with $m=2$}
\begin{tabular}{llcccc}
\toprule
Graph & Sampler & Energy ($\mu\pm\sigma$) & Cover$^{\downarrow}$ ($\mu\pm\sigma$) & Best$^{\downarrow}$ & Uncov.$^{\downarrow}$ ($\mu\pm\sigma$) \\
\midrule
\multirow{4}{*}{BA64}
 & ULA   & $1.327\pm0.212$ & $45.43\pm2.70$ & 38 & $0.115\pm0.049$ \\
 & FULA  & $0.842\pm0.191$ & $39.56\pm3.31$ & $\mathbf{31}$ & $\mathbf{0.048\pm0.034}$ \\
 & MALA  & $1.431\pm0.212$ & $45.43\pm2.73$ & 38 & $0.128\pm0.051$ \\
 & MAFLA & $0.872\pm0.225$ & $\mathbf{38.61\pm3.56}$ & $\mathbf{31}$ & $0.057\pm0.040$ \\
\midrule
\multirow{4}{*}{BA256}
 & ULA   & $1.720\pm0.178$ & $181.20\pm5.84$ & 166 & $0.194\pm0.040$ \\
 & FULA  & $0.942\pm0.119$ & $166.96\pm7.23$ & 148 & $0.059\pm0.020$ \\
 & MALA  & $1.780\pm0.166$ & $181.12\pm5.80$ & 165 & $0.203\pm0.040$ \\
 & MAFLA & $0.905\pm0.115$ & $\mathbf{163.96\pm7.67}$ & $\mathbf{140}$ & $\mathbf{0.053\pm0.020}$ \\
\midrule
\multirow{4}{*}{BA512}
 & ULA   & $1.827\pm0.145$ & $360.52\pm7.69$ & 336 & $0.216\pm0.031$ \\
 & FULA  & $1.010\pm0.094$ & $345.06\pm9.63$ & 320 & $0.067\pm0.016$ \\
 & MALA  & $1.865\pm0.137$ & $361.07\pm7.81$ & 338 & $0.222\pm0.032$ \\
 & MAFLA & $0.967\pm0.082$ & $\mathbf{342.12\pm9.70}$ & $\mathbf{316}$ & $\mathbf{0.059\pm0.014}$ \\
\midrule
\multirow{4}{*}{BA1024}
 & ULA   & $1.892\pm0.116$ & $716.20\pm11.65$ & 679 & $0.228\pm0.025$ \\
 & FULA  & $1.094\pm0.076$ & $703.16\pm12.64$ & 670 & $0.079\pm0.013$ \\
 & MALA  & $1.921\pm0.108$ & $716.20\pm11.76$ & 684 & $0.233\pm0.026$ \\
 & MAFLA & $1.027\pm0.065$ & $\mathbf{696.98\pm12.99}$ & $\mathbf{655}$ & $\mathbf{0.066\pm0.011}$ \\
\bottomrule
\end{tabular}
\end{subtable}
\end{table*}

\paragraph{Discussion.} The empirical gains observed in both MaxCut and Vertex Cover can be attributed to two complementary aspects of our framework: the use of $\alpha$-stable fractional dynamics and the incorporation of an acceptance correction. Fractional Langevin dynamics introduce heavy-tailed noise, which enables occasional long-range moves in the continuous relaxation space. Compared to Gaussian perturbations, this behavior facilitates exploration across energy barriers and helps the sampler escape local minima induced by the highly nonconvex combinatorial objectives. This effect becomes increasingly pronounced as graph size grows, where the energy landscape exhibits a large number of metastable configurations.

While fractional dynamics enhance exploration, they can also lead to unstable updates or overly aggressive proposals when used without correction. The acceptance mechanism in MAFLA mitigates this issue by selectively rejecting unlikely transitions, thereby stabilizing the sampling process and preserving the target distribution. Empirically, this combination yields consistent improvements over both Gaussian
samplers and non-adjusted fractional methods, achieving better trade-offs between exploration, stability, and solution quality across problem instances.

\end{document}